\documentclass{article}
\usepackage{amsmath} % assumes amsmath package installed
\usepackage{amsthm}
\usepackage{amssymb}  % assumes amsmath package installed
\usepackage{graphicx}
\usepackage{booktabs}
\usepackage{multirow}
\usepackage{array}
\usepackage{makecell}
\usepackage{wrapfig}
\usepackage{caption}

\usepackage{appendix}

\usepackage[final]{corl_2025} % Uncomment for the camera-ready ``final'' version.

\usepackage{tikz}
\usetikzlibrary{angles,quotes,calc}
\usepackage{mdframed} %
\usepackage{listings}%
\definecolor{light-gray}{gray}{0.95}

\usepackage{tcolorbox}

\newtheorem{theorem}{Theorem}
\newtheorem*{utheorem}{Theorem}

\usepackage{xcolor}
\definecolor{flodarkpurple}{rgb}{0.694,0.016,0.055} 

\usepackage[strict]{changepage}
\usepackage{framed}
\title{Real-Time Out-of-Distribution~Failure~Prevention via Multi-Modal Reasoning}

\newcommand{\authorhref}[3][black]{\href{#2}{\textcolor{#1}{#3}}}
\author{
    \bfseries
    \authorhref{https://milanganai.github.io/}{Milan Ganai}\textsuperscript{1},\,\,
    \authorhref{https://rohansinha.nl/}{Rohan Sinha}\textsuperscript{1,†},\,\,
    \authorhref{http://agiachris.github.io/}{Christopher Agia}\textsuperscript{1,†},\,\,
    \authorhref{https://danielpmorton.github.io/}{Daniel Morton}\textsuperscript{1},\\ \bfseries
    \authorhref{https://ieeexplore.ieee.org/author/37089761562}{Luigi Di Lillo}\textsuperscript{2},\,\,
    \authorhref{https://profiles.stanford.edu/marco-pavone}{Marco Pavone}\textsuperscript{1,3}\\\vspace{-6pt}\\
    \textsuperscript{1}\href{https://www.stanford.edu/}{\textcolor{flodarkpurple}{Stanford University}}, \textsuperscript{2}\href{https://www.swissre.com/}{\textcolor{flodarkpurple}{Swiss Re}}, \textsuperscript{3}\href{https://www.nvidia.com/en-us/research/}{\textcolor{flodarkpurple}{NVIDIA Research}}
    \vspace{-8pt}
}

\begin{document}
\maketitle
\begingroup
\renewcommand\thefootnote{\textsuperscript{†}}
\footnotetext{Equal contribution.}
\renewcommand\thefootnote{\textsuperscript{}}
\footnotetext{Correspondence to Milan Ganai: \texttt{mganai@stanford.edu}}
\endgroup

%===============================================================================

\vspace{-22pt}
\begin{figure}[h]
    \centering
    \includegraphics[width=\textwidth]{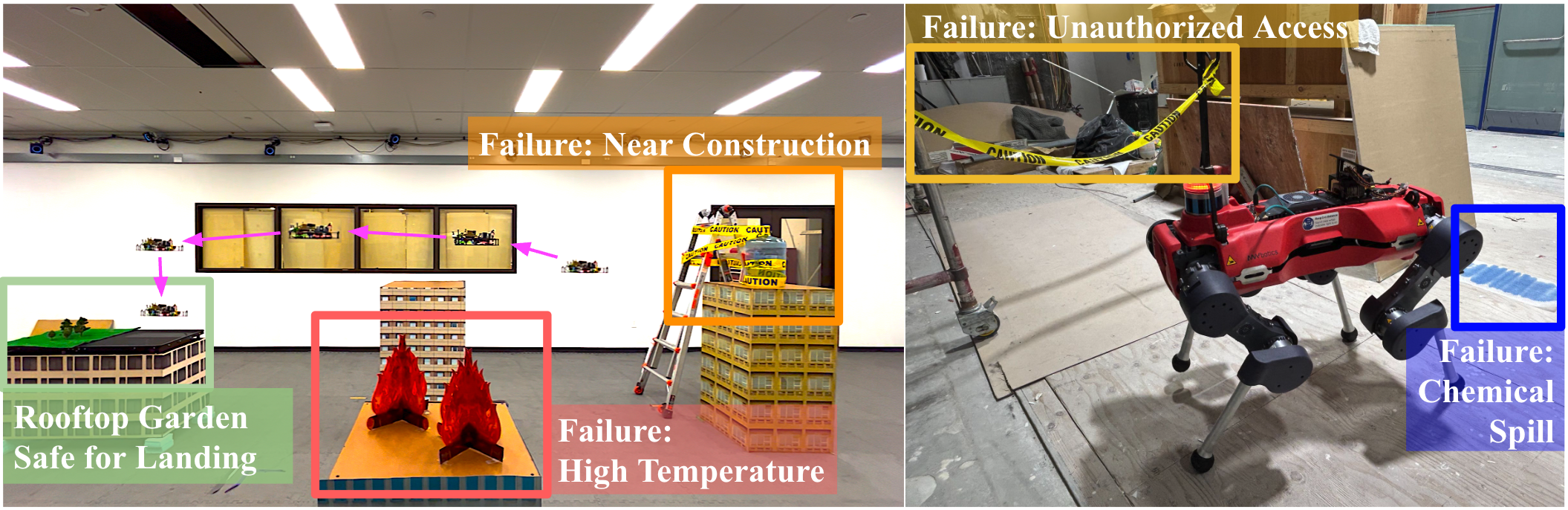}
    \caption{\footnotesize We present \texttt{FORTRESS}, a multi-modal reasoning framework that prevents catastrophic robot failures in safety-critical, out-of-distribution scenarios. \texttt{FORTRESS} incorporates the generalist knowledge of foundation models within a control-theoretic optimization framework to synthesize safe and dynamically feasible fallback plans, enabling robust responses to unforeseen hazards during open-world deployment. For example: (left) executing an emergency rooftop landing, or (right) avoiding semantically unsafe regions in construction zones.
    }
    \label{fig:herofig}
\vspace{-12pt}
\end{figure}

\begin{abstract}
    While foundation models offer promise toward improving robot safety in out-of-distribution (OOD) scenarios, how to effectively harness their generalist knowledge for real-time, dynamically feasible response remains a crucial problem. We present \texttt{FORTRESS}, a joint reasoning and planning framework that generates semantically safe fallback strategies to prevent safety-critical, OOD failures. At a low frequency under nominal operation, \texttt{FORTRESS} uses multi-modal foundation models to anticipate possible failure modes and identify safe fallback sets. When a runtime monitor triggers a fallback response, \texttt{FORTRESS} rapidly synthesizes plans to fallback goals while inferring and avoiding semantically unsafe regions in real time. By bridging open-world, multi-modal reasoning with dynamics-aware planning, we eliminate the need for hard-coded fallbacks and human safety interventions. \texttt{FORTRESS} outperforms on-the-fly prompting of slow reasoning models in safety classification accuracy on synthetic benchmarks and real-world ANYmal robot data, and further improves system safety and planning success in simulation and on quadrotor hardware for urban navigation. Website can be found at \textcolor{flodarkpurple}{https://milanganai.github.io/fortress}.
\end{abstract}

\keywords{Multi-modal Reasoning, Semantic Safety, OOD Reliability} 

%===============================================================================
\vspace{-0.2in}
\section{Introduction}
\vspace{-0.1in}
    
Across various industries, autonomous robots~\citep{Council2025Waymo, BMWGroup2024HumanoidRobots, SaudiAramcoAUV2024, GreeceDronesFire2024, CAAInfrastructureDrones2024, EVMagazineAutonomousAgriculture2024, SixthToneDronesDelivery2024} are expanding their deployment regions from structured, localized settings to unstructured, open-world environments~\citep{drummond2006open}. 
In the process, they increasingly encounter out-of-distribution (OOD) failure scenarios---situations that differ significantly from the system’s design assumptions and training data, leading to degraded performance, unsafe behavior, or even complete system malfunction~\citep{lu2024recent, sinha2022system}. To ensure the reliability and safety of autonomous robots in expanding operations, it is crucial to: adaptively identify goals for fallback strategies, accurately discern unsafe regions where entering could lead to OOD failures, and quickly generate executable fallback plans that are semantically safe (i.e. contextually appropriate plans that circumvent conceptually hazardous regions).

\begin{figure}[t!]
    \centering
    \includegraphics[width=\textwidth]{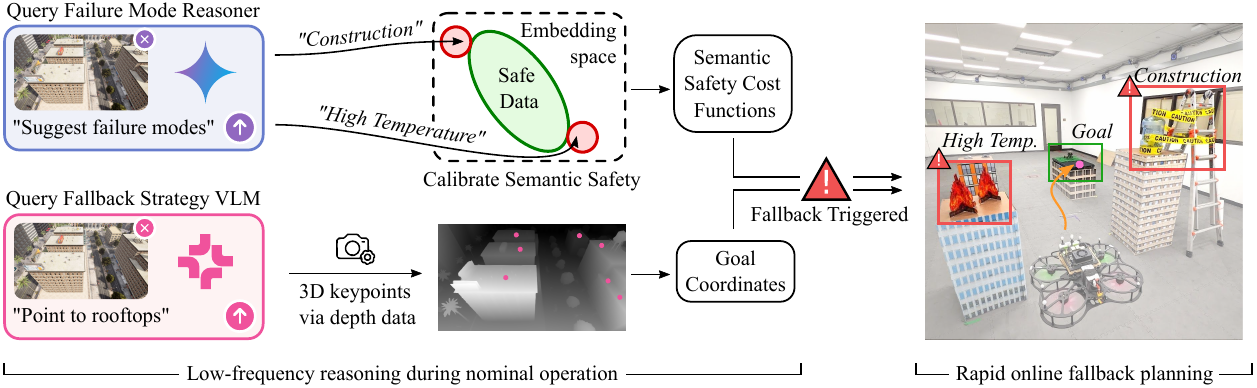}
    \caption{\footnotesize \textbf{Overview of \texttt{FORTRESS}:} Our proposed framework prevents OOD failures by reasoning about concrete goal locations based on semantic descriptions of fallback strategies, anticipating failure modes, and constructing cost functions that identify semantically unsafe regions at a low frequency during the nominal trajectory. When a safety response is triggered, the algorithm rapidly produces semantically safe fallback plans.}
    \label{fig:overview}
\vspace{-0.2in}
\end{figure}

Pre-trained foundation models, like Large Language Models (LLMs) and Vision-Language Models (VLMs)~\citep{bommasani2021opportunities}, has been leveraged to understand OOD scenarios~\citep{elhafsi2023semantic, sinha2022system} and reason in a zero-shot manner about high-level semantically appropriate responses~\citep{sinha2024real}. However, these high-level descriptions of safety interventions can be limiting: while a model might suggest a drone to ``land in a parking lot" to avoid a building on fire, this could lead to semantically unsafe plans that get too close to first responders involved at the scene. These models are inherently dynamics-agnostic, and attempts to integrate them~\citep{firoozi2023foundation} into end-to-end control pipelines can compromise robustness~\citep{pham2023combined, radford2021learning}. Moreover, the models' high inference latency~\citep{jabbour2024generative} renders them unsuitable for real-time querying for granular planning in safety-critical moments. Approaches resort to rigid, hard-coded fallback regions~\citep{sinha2024real} and human interventions for safety~\citep{santos2024updating}, which are impractical in open-world environments where potential failure scenarios are vast. The challenge is to use foundation model reasoning for the real-time generation of fallback plans that are semantically safe and dynamics-feasible.

We demonstrate how to leverage foundation models to identify fallback goals, anticipate failures, and infer semantically dangerous regions. By performing slow reasoning in advance, we expedite the generation of semantically safe fallback plans. Particularly, we introduce \texttt{FORTRESS}, a framework for OOD \underline{F}ailure Preventi\underline{o}n in \underline{R}eal Time by Genera\underline{t}ing and \underline{Re}asoning about Fallback \underline{S}trategie\underline{s} (overview in Fig.~\ref{fig:overview}). \texttt{FORTRESS} is a multitiered, slow-fast hierarchical approach: At a low frequency or offline, the algorithm (i) uses VLMs to translate abstract semantic strategies into suggestions of physical fallback goals and (ii) constructs cost functions to quickly identify semantically unsafe state spaces, entering which can result in OOD failures by calibrating embedding models with failure modes from foundation model reasoners. Once a runtime monitor triggers a response, \texttt{FORTRESS} (iii) enables existing motion planners to rapidly infuse semantic safety into fallback trajectories.

\vspace{-0.1in}
\section{Related Works}\label{sec:RW}
\vspace{-0.1in}

Safe control for autonomous systems has been studied through various formal methods~\citep{lindemann2024formal, garg2024learning, dawson2023safe}. Hamilton–Jacobi (HJ) reachability~\cite{hjreachabilityoverview, ganai2023iterative} can provide worst‑case reach‑avoid guarantees~\citep{chen2018hamilton, fisac2015reach, ganai24Hamilton}, and Control Barrier Functions (CBFs) certify control invariance~\citep{cbfsurvey}. They can be used as online safety filters to monitor the system behavior and invoke corrective controls when violations are imminent~\citep{hsu2023safety, chen2021backup, so2024train}. However, these tools require precise models of the environment and explicitly defined failure regions which are not always readily available in open-world settings.

Foundation models~\citep{bommasani2021opportunities,firoozi2023foundation}, including LLMs and VLMs, have enabled semantic planning in robotics~\citep{ahn2022can,singh2023progprompt}, such as in quadrotors~\citep{saviolo2024unifying,chen2023typefly}, manipulation~\citep{kim2024openvla,huang2023voxposer}, and self-driving cars~\citep{gao2024survey}. However, the models' high latency hinders real‑world deployment~\citep{jabbour2024generative}. One approach~\citep{santos2024updating} mitigates this by feeding language constraints to a fast VLM object detector~\citep{minderer2023scaling} and updating reachability–based online safety filters~\citep{bajcsy2019efficient, herbert2019reachability}, but the constraints are obtained from human intervention.

Real‑time recovery planning helps ensure safety in autonomous systems once a response is triggered. Current methods include game‑theoretic autonomous vehicle trajectory repair~\citep{wang2024interaction}, LLM‑based semantic trajectory fixes~\citep{lin2024drplanner}, satisfiable modulo theory and reachability analysis for traffic-rule-compliant repairs~\citep{lin2024traffic}, and perception‑failure detection with a trained safety monitor that triggers learned recovery plans~\citep{chakraborty2024system}. However, these approaches only offer localized trajectory repairs and lack dynamic real‑time generation of new fallback goals and paths, a crucial capability when the nominal path and goal suddenly become unsafe in open-world environments.

Furthermore, real-world systems face OOD scenarios that significantly deviate from training distributions. Some methods to improve system robustness include training for distributionally robust optimization~\citep{lin2022distributionally, hejna2024re} and model adaptation~\citep{chen2023adapt, hoffman2018cycada}. Recent research works investigate OOD detection~\citep{salehi2021unified, ruff2021unifying, liu2020energy, lakshminarayanan2017simple} paired with fallback controllers~\citep{gupta2024detecting, sinha2023closing, richter2017safe}. Foundation models employed in a zero‑/few‑shot manner~\citep{wortsman2022robust} can understand semantic anomalies~\citep{elhafsi2023semantic} and enable closed‑loop failure detection~\citep{elhafsi2023semantic, AgiaSinhaEtAl2024} along with fallback set selection~\citep{sinha2024real}. However, these rely on handcrafted policies or static fallback sets, which are inadequate in preventing OOD failures.

Despite significant advances in robotics planning and control, existing methods fall short in dynamic, open-world conditions. By relying on precise environment models, human interventions, and rigid fallback plans and/or goals, they leave a critical gap in unstructured settings when the nominal path's deployment becomes unsafe. We need systems that adaptively generate real‑time fallback strategies, ensuring recoveries are semantically safe in uncertain real‑world environments.

\vspace{-0.15in}
\section{Problem Formulation}\label{sec:PF}
\vspace{-0.15in}

We investigate robots with known discrete-time system dynamics given by $x_{t+1} = f(x_t, u_t)$, for state $x \in \mathcal{X} \subset \mathbb{R}^n$ and control $u \in \mathcal{U} \subset \mathbb{R}^m$. While we can control the robot's motion in its physical space, we aim to imbue understanding of semantic safety in the responses to OOD failures not captured in the robot's state -- like avoiding a roof on fire. To build semantic awareness, we assume access to safe semantic descriptions $\Omega_s:=\{\omega_i\}_{i=1}^N$---for example, state descriptions with ``buildings in city" for drones or ``trees on median strip" for autonomous vehicles. This is practical since systems have vast descriptions of logs from successful deployments and training data. We also assume we have a set of semantic descriptions of high-level fallback strategies $\Sigma := \{ \sigma_1, \sigma_2, \dots \}$, like ``empty flat roofs" for drone landing or ``vacant road shoulders" for vehicle parking. While engineers or language models can easily generate these general strategy descriptions from prior data, we aim to realize the physical execution of such strategies. The robot has function $\texttt{nearby}(x,l)$ that returns a state description of all concepts within distance $l$ of coordinate $x$, which onboard RGBD cameras and pre-trained object detectors (like OWL-ViT~\citep{minderer2022simple}, OWLv2~\citep{minderer2023scaling}, and YOLOv8~\citep{sohan2024review}) can rapidly provide. We also have cost functions $\theta_c$ to capture collision hazards, where $\theta_c(x)>0$ if and only if $x$ is within distance $l_c$ of a physical obstacle, which can be obtained from depth sensors. Following~\citep{sinha2024real, chen2021backup, so2024train}, we assume that the robot is equipped with a runtime monitor, which is a system that raises an alarm when it detects anomalies or potential hazards in the robot's environment. The monitor's alarm prompts us to reevaluate the robot's plan and identify an appropriate response.

We address the challenge of generating semantically safe fallback plan responses. In open-world environments, novel semantically unsafe concepts, with descriptions referred to as $\Omega_d$, can appear anywhere. For example, while ``person" and ``ladder" are individually safe, ``person on a ladder" poses a ``Worker Injury" hazard for the ANYmal robot, as illustrated in Fig.~\ref{fig:anymal}. These unsafe concepts may be related or even unrelated to what triggered the fallback. Robots generally have limited access to data on experiencing and responding to failure. So if a robot encounters state descriptions in $\Omega_d$, we consider these OOD failures relative to our safe training data, $\Omega_s$. 

\emph{Objective:} We aim to design an algorithm to identify, plan, and execute semantically safe fallback behaviors when anomalous conditions render the robot's original task unsafe. The robot must:

\textbf{M1}: Be prepared with relevant concrete fallback strategy goal locations aligned with semantic strategy idea set $\Sigma$, when the original nominal path is rendered infeasible or unsafe.

\textbf{M2}: Bootstrap its understanding of semantically unsafe states that would result in OOD failures, i.e., states $x$ where $\texttt{nearby}(x,l)\in \Omega_d$, using our plentiful safe state descriptions $\Omega_s$.

\textbf{M3}: Rapidly generate safe plans that implement a fallback strategy while satisfying the semantic safety constraints and physical constraints.

\vspace{-0.1in}
\section{Proposed Approach}\label{sec:PA}
\vspace{-0.1in}

We describe our framework \texttt{FORTRESS}, which generates and reasons about semantically safe fallback strategies to prevent OOD failures (Fig.~\ref{fig:overview}). Our algorithm has three key components: (i) constructing physical fallback goal locations from abstract semantic strategy ideas (Section~\ref{sec:fallback}), (ii) rapidly inferring semantically unsafe regions (Section~\ref{sec:FailReason}), and (iii) employing foundation model reasoning with worst-case analysis to rapidly generate semantically safe fallback plans (Section~\ref{sec:SafePlanGen}).

\vspace{-0.1in}
  \subsection{Generating potential fallback strategy sets \textbf{(M1)}}\label{sec:fallback}
\vspace{-0.1in}

\begin{wrapfigure}{l}{0.4\linewidth}
        \vspace{-4pt} % Optional: tweak to snug against text above
        \includegraphics[width=\linewidth]{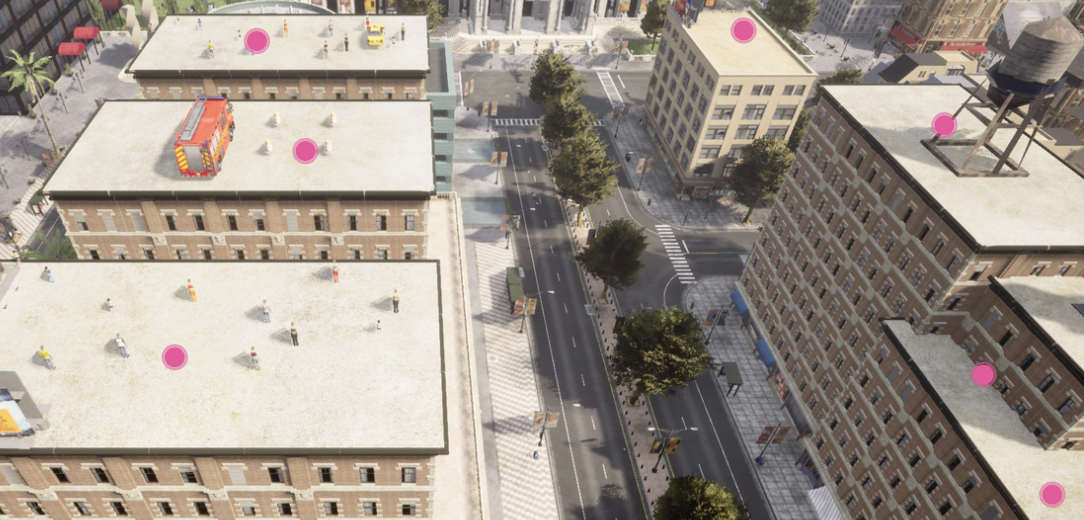}
        \caption{\footnotesize For a drone agent in the CARLA simulator, \texttt{FORTRESS} identifies goal locations for semantic fallback strategy description ``empty, horizontal building roofs" using VLMs like Molmo. VLMs produce safe goals, but sometimes they also generate unsafe ones (e.g. near people, cars, etc) and are not dynamics-aware (e.g. tight landing spots require complex maneuvering). }
        \label{fig:molmo}
        \vspace{-6pt} % Optional: tweak to snug against text below
\end{wrapfigure}

A crucial aspect in executing fallback strategies is generating clear goals for contingency plans when the original plan's goal is unattainable. We leverage the abstract semantic fallback characteristics $\Sigma$ defined in Section~\ref{sec:PF}. These fallback goals serve various purposes, such as waiting for external dangers to subside, recalibrating perception systems, enabling controlled hardware/software maintenance, providing extra time to reassess the environment, creating more robust mitigation strategies, or receiving human intervention. To transform semantic descriptions into physical fallback goal locations, we utilize the general-purpose reasoning of VLMs. In particular, we employ the VLM Molmo trained on the PixMo dataset~\citep{deitke2024molmo} to analyze a robot’s scene image with a query based on a strategy description $\sigma_i\in \Sigma$. The VLM outputs pixel coordinates for potential fallback points (shown in Fig.~\ref{fig:molmo}), which are then converted into 3D global coordinates using depth information and camera intrinsics, to form a set of goal locations $\mathcal{G}_i$. VLM queries can incur several seconds of latency, depending on model size and token limits. Therefore we propose mitigating delays during critical moments by preemptively querying the VLM at a low frequency during normal operations (see Table~\ref{tab:Times} for times) to identify potential fallback locations, caching their corresponding 3D global coordinates, and loading these locations immediately when a fallback response is needed.

\vspace{-0.1in}
\subsection{Reasoning about semantic safety constraints \textbf{(M2)}}\label{sec:FailReason}
\vspace{-0.1in}

Another key contribution is an approach for rapid and efficient reasoning to identify semantically unsafe state spaces. While the robot has access to many descriptions of safe, nominal data $\Omega_s$, it faces the challenge of identifying semantically unsafe regions that are naturally OOD compared with $\Omega_s$. We propose to anticipatively enumerate a set of high-level failure modes $\Phi$ (e.g. ``Near Human," ``Turbulence," ``High Temperature"). Then, we construct functions that determine the safety of a description semantically by measuring if a new state description is anomalously close to a high-level failure mode $\phi \in \Phi$, relative to the safe observations in $\Omega_s$. Our approach is based on the insight that despite the potentially infinite variations in details of dangerous scenarios, most semantically unsafe scenarios are associated with a limited number of abstract, anticipated modes.

One approach to anticipate the abstract semantic failure modes that the environment may produce is to manually identify them, but this does not scale and adapt well to novel OOD settings. We propose using pre-trained foundation models offline to reason about failure modes capturing broad classes of potential semantically unsafe scenarios a robot may encounter in an environment. We create a structured prompt which has a general description of the environment, an image (if the reasoning model can process them), and a query for a list of relevant semantic failure modes. We extract the phrases of failure modes set $\Phi$ from the response as shown in the left part of Fig.~\ref{fig:anymal}. This step of forecasting failure modes functionally corresponds to the \emph{Failure Mode and Effects Analysis} stage of describing potential failures in a system for robust risk assessment~\citep{rausand2003system}.

\begin{wrapfigure}{r}{0.63\linewidth}
        \vspace{-8pt} % Optional: tweak to snug against text above
        \includegraphics[width=\linewidth]{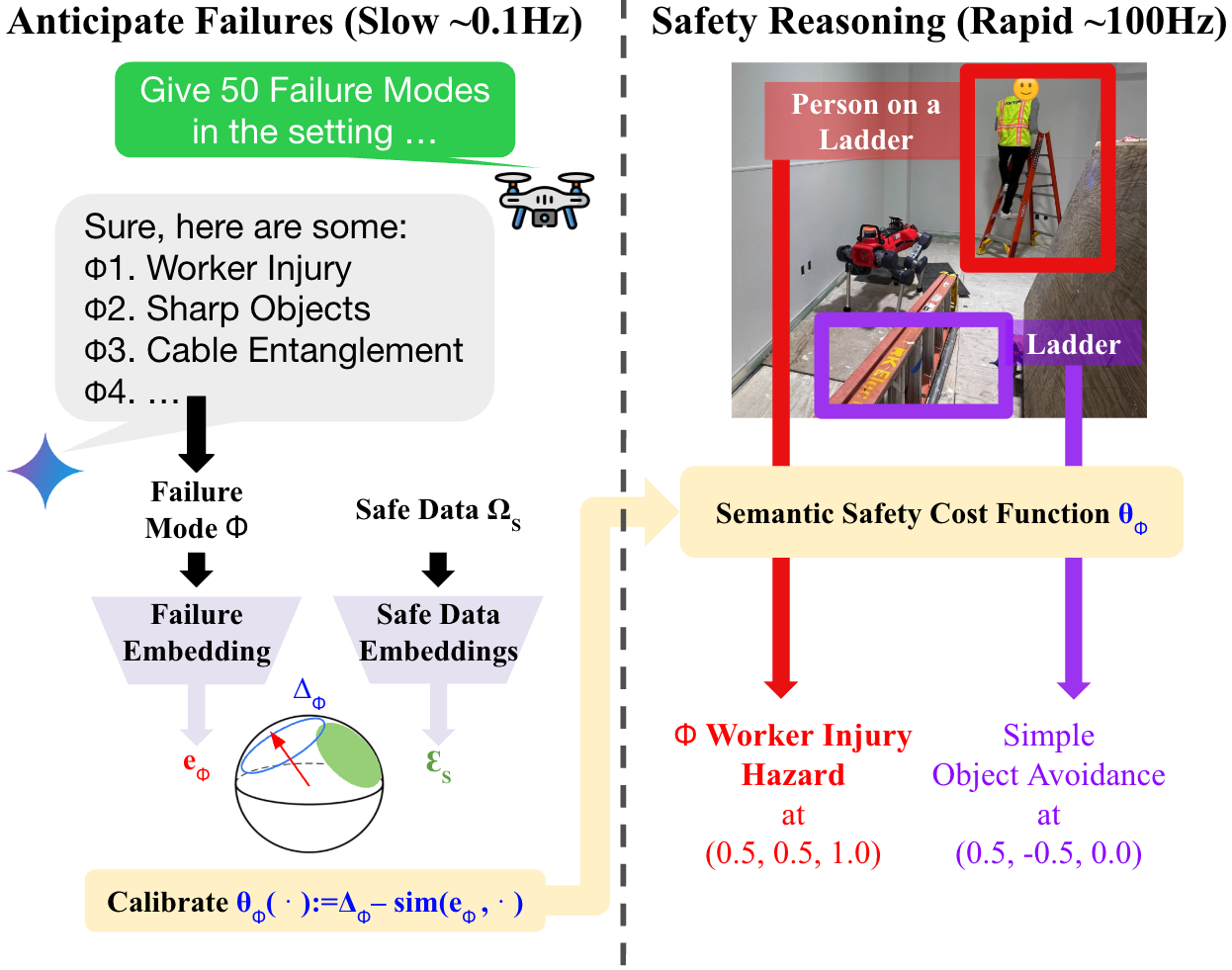}
        \caption{\footnotesize \texttt{FORTRESS} employs foundation model reasoners to anticipate failure modes. It then calibrates thresholds in the embedding model space to determine if new state descriptions more similar to failure modes than safe data $\Omega_s$. During safety-critical moments, the semantic safety cost functions rapidly identify physical unsafe state regions during an ANYmal robot's deployment. \texttt{FORTRESS} differentiates the safety of a ladder from a person standing on one, anticipating worker injuries without encountering failures in $\Omega_s$.}
        \label{fig:anymal}
        %\vspace{-4pt} % Optional: tweak to snug against text below
\end{wrapfigure}

To quantify the relationship of state description $\omega$ to some failure mode $\phi\in\Phi$, we leverage semantic understanding capabilities of a text embedding model $\texttt{Embed}$. During the nominal trajectory or offline, we construct a set of safe embedding vectors $\mathcal{E}_s:=\{ \texttt{Embed}(\omega_s), \forall \omega_s \in \Omega_s \}$ from safe state observation set $\Omega_s$ and obtain each failure mode's embedding $\mathbf{e}_\phi := \texttt{Embed}(\phi)$. We use the cosine similarity-based function 
\begin{equation*}
\texttt{sim}(\mathbf{e}_i, \mathbf{e}_\phi) := 1-\frac{\mathbf{e}_i \cdot  \mathbf{e}_\phi}{||\mathbf{e}_i|| \ || \mathbf{e}_\phi||}
\end{equation*}
to measure semantic affinity of a state description embedding vector $\mathbf{e}_i$ to $\phi$. We identify if a state description is significantly closer to failure mode $\phi$ compared with safe data $\Omega_s$ and therefore is unsafe, by calibrating a threshold $\Delta_\phi$ using conformal prediction on the $\alpha\in(0,1)$ quantile of safe data:
\begin{equation}
    \Delta_\phi \ := \ \sup\bigl\{\delta\in\mathbb{R} \ : \ \big|  \{e_s\in \mathcal{E}_s \ : \ \texttt{sim}(e_s , \ e_\phi) \geq \delta\} \big| \ \geq \ (1-\alpha) N \bigr\},
\label{eq:calib}
\end{equation}
which is the largest real scalar that lower bounds at least $\lceil(1-\alpha) N\rceil$ safe data similarity scores \citep{angelopoulos2024Conformal}.

To quantify the hazard of regions in the state space with respect to the robot, we propose cost functions based on conformity scores~\cite{angelopoulos2024Conformal}. We use the function $\texttt{nearby}$ (Section~\ref{sec:PF}) and failure mode $\phi$ to construct semantic safety cost functions that analyze concepts within radius $l_\phi$ of state $x$: 
\begin{equation}
    \theta_\phi(x):= \Delta_\phi - \texttt{sim}(\texttt{Embed}(\texttt{nearby}(x, \ l_\phi)), \ e_\phi),
\label{eq:cost}
\end{equation}
where the robot at $x$ is hazardous if $\theta_\phi(x)>0$. Intuitively, any state description's embedding within $\Delta_\phi$ semantic similarity of failure mode embedding $e_\phi$ is classified as semantically unsafe (right part of Fig~\ref{fig:anymal}). We use the notation $h\in\Phi\cup\{c\}$ to indicate all failure modes and the collision hazard.

\vspace{-0.1in}
\subsection{Synthesizing semantically safe fallback plans \textbf{(M3)}}\label{sec:SafePlanGen}
\vspace{-0.1in}

The safety-critical, rapid planning layer in \texttt{FORTRESS} is responsible for generating a semantically safe path to a goal. We develop a control theoretic optimization framework that reasons about worst-case failure modes using the tools we developed in previous sections to synthesize the plans.

We construct a Reach-Avoid problem of entering a region $\mathcal{B}_{\rho}(g):=\{x: ||x-g||\leq\rho\}$ (a ball of radius $\rho$ centered around some goal $g$ from Section~\ref{sec:fallback}) while avoiding semantically unsafe regions in the state space. We leverage the semantic safety cost functions from~\eqref{eq:cost} for all failure modes in $\Phi$. We obtain trajectory plan $\tau=x_{\{1:T\}}$ with horizon $T$ that starts at the robot's current location $b$ and reaches goal location $g$ by solving the following, where $\mathbb{N}^{\leq T}$ indicates the set $\{1,2,...T\}$:
\begin{equation}
\begin{aligned}
&&\quad \text{s.t.} \quad & x_1 = b \\
\tau^* &= \arg\min_{\tau} \max_{h \in \Phi\cup \{c\}} \max_{x \in \tau} \ \theta_h (x)&& \exists k \in \mathbb{N}^{\leq T},\quad \{x_i|k\leq i\leq T\} \subseteq \mathcal{B}_\rho(g)\\
&&& \exists u\in \mathcal{U},\forall j\in\mathbb{N}^{\leq T-1} \quad  x_{j+1} = f(x_j, u) ,
\end{aligned}
\label{eq:optim}
\end{equation}
where $\tau^*$ is the optimal trajectory. We define $\Theta^*:= \max_{h \in \Phi\cup \{c\}} \max_{x \in \tau^*} \theta_h (x)$ as the minimaximax objective value. This optimization ensures the trajectory remains semantically and physically safe (i.e. $\theta_h(x)\leq0$) if possible by minimizing worst-case influences of failures along the entire trajectory. The first constraint forces the trajectory to start at the robot's current location, the second requires the trajectory will reach and remain indefinitely near the goal, and the third guarantees the trajectory is dynamically feasible. This optimization structure is related to HJ reachability's~\citep{bajcsy2024human} game-theoretic optimization by scoping out the ``adversary's" strategies and considering worst-case bounded ``disturbances" that the environment may produce (i.e. potential hazards $\Phi\cup\{c\}$). 

We employ a combination of motion planning and path tracking to find an approximate solution to optimization of~\eqref{eq:optim} and compute the executable controls. We use Rapidly exploring Random Trees (RRT)~\citep{lavalle1998rapidly} to plan a trajectory to the goal region using the cost functions $\theta_h(x)$ to guide and invalidate unsafe states during the search procedure. Then, we employ Model Predictive Control (MPC) or Linear Quadratic Regulator (LQR) to obtain controls for tracking the RRT plan. By inflating the distance thresholds $l_h, \ h\in\Phi\cup\{c\}$ for the cost functions outlined in Section~\ref{sec:FailReason} and ensuring the incremental distance of the planner is small enough to account for the error of the path tracking algorithm, we can leverage results similar to~\citep{herbert2017fastrack, luders2010bounds} and prove the executed trajectory will align with the safety and reachability requirements (formalization and proof in Appendix Section~\ref{sec:PlanningProof}).

Once entering the goal region $\mathcal{B}_\rho(g)$, the robot executes controls to ensure control invariance of the fallback set (e.g. initiate landing, hovering in place, etc). If the planner cannot identify a safe trajectory (i.e. $\Theta^*>0$), we iterate through strategies $\sigma_i\in \Sigma$ from Section~\ref{sec:fallback}, selecting a new goal from $\mathcal{G}_i$ and recomputing the optimization of ~\eqref{eq:optim}, until a safe trajectory is produced (i.e., $\Theta^*\leq 0$). During planning, we can log which hazards in $\Phi\cup \{c\}$ prevented the implementation of any fallback strategy, providing interpretable explanations for why certain strategies in $\Sigma$ could not be executed.

\begin{table}[t!]
\centering
{\footnotesize
\begin{tabular}{@{}l@{\hspace{2pt}}l@{\hspace{4pt}}ccc@{\hspace{8pt}}ccc@{\hspace{8pt}}ccc@{}}
\toprule
 & \multirow{+2}{*}{\makecell[l]{(Best) Embed.}} & \multicolumn{3}{c}{Boat (Synthetic)} & \multicolumn{3}{c}{Vehicle (Synthetic)} & \multicolumn{3}{c}{ANYmal \textbf{(Hardware)}}  \\
\cmidrule(lr){3-5} \cmidrule(lr){6-8} \cmidrule(lr){9-11}
Method & Model & TPR & TNR & Bal. Acc. & TPR & TNR & Bal. Acc. & TPR & TNR & Bal. Acc. \\
\midrule

 & Voyage AI  & \underline{0.97} & \underline{0.97} & \textbf{0.97} & 0.81 & 0.98 & 0.89 & 0.59 & 0.87 & 0.73 \\
\multirow{-2}{*}{\makecell[l]{10 modes}} & Qwen2  & 0.78 & 0.66 & 0.72 & 0.65 & 0.90 & 0.78 & 0.82 & \underline{0.98} & \textbf{0.90} \\
\multirow{-2}{*}{\makecell[l]{(Ours)}} & OpenAI & 0.72 & 0.90 & 0.81 & 0.83 & 0.98 & \textbf{0.90} & 0.51 & 0.82 & 0.66 \\
\midrule

 & Voyage AI & 0.34 & 0.78 & 0.56 & 0.05 & 0.96  & 0.50 & 0.02 & \underline{0.98} & 0.50\\
\multirow{-2}{*}{\makecell[l]{1 mode}} & Qwen2  & 0.92 & 0.10 & 0.51 & 0.05 & \underline{1.00} & 0.53 & 0.65 & 0.60 & 0.63\\
\multirow{-2}{*}{\makecell[l]{``Safe"}} & OpenAI  & 0.61 & 0.47 & 0.54 & 0.25 & 0.86 & 0.55 & 0.15 & 0.97 & 0.56 \\
\midrule
GPT-4o & N/A - prompting  & 0.79 & 0.73 & 0.76 & \underline{1.00} & 0.23 & 0.61 & 0.89 & 0.40 & 0.64 \\
o3-mini & N/A - prompting  & 0.79 & 0.87 & 0.83 & 0.73 & 0.93 & 0.83 & 0.95 & 0.80 & 0.87 \\
o4-mini & N/A - prompting  & 0.72 & 0.79 & 0.76 & 0.88 & 0.85 & 0.86 & \underline{1.00} & 0.70 & 0.85 \\
\bottomrule
\end{tabular}
}
\caption{\footnotesize Accuracy of our approach for calibrating embedding model-based cost functions to detect semantically unsafe descriptions on synthetic datasets and data collected from ANYmal hardware exploring a room under construction. Performance metrics are measured with True Negative Rate (TNR), True Positive Rate (TPR), and Balanced Accuracy. Results are shown for \texttt{FORTRESS} with $10$ failure modes, ablation baselines that compute similarity distance to 1 mode called ``Safe," and on-the-fly prompting of slow reasoning models.}
\label{tab:10acc}
%\vspace{-0.3in}
\vspace{-0.15in}
\end{table}

%\vspace{-10pt}

%\vspace{-0.15in}
\vspace{-0.1in}
\section{Experiments}\label{sec:E}
\vspace{-0.1in}
%\vspace{-0.15in}

We conduct various experiments to test three hypotheses of our approach \texttt{FORTRESS}:

\textbf{H1}: By reasoning about the safety of a scene using text embedding models and a limited set of failure modes, we can identify OOD failures (i.e. semantically unsafe descriptions) more accurately than on-the-fly prompting of slow foundation models reasoning about safety of a description.

\vspace{-0.05in}

\textbf{H2}: \texttt{FORTRESS} can leverage foundation model reasoners to automatically anticipate failure modes that semi-monotonically improve classification accuracy using embedding-based cost functions. 

\vspace{-0.05in}

\textbf{H3}: By performing rapid constraints and planning computation when requested for a fallback plan, we can generate fallback plans in \emph{real time} that have improved planning success and safety compared to baseline approaches that hard-code fallbacks and/or perform na\"ive object avoidance.

\begin{figure}[t!]
    \centering
    \includegraphics[width=\linewidth]{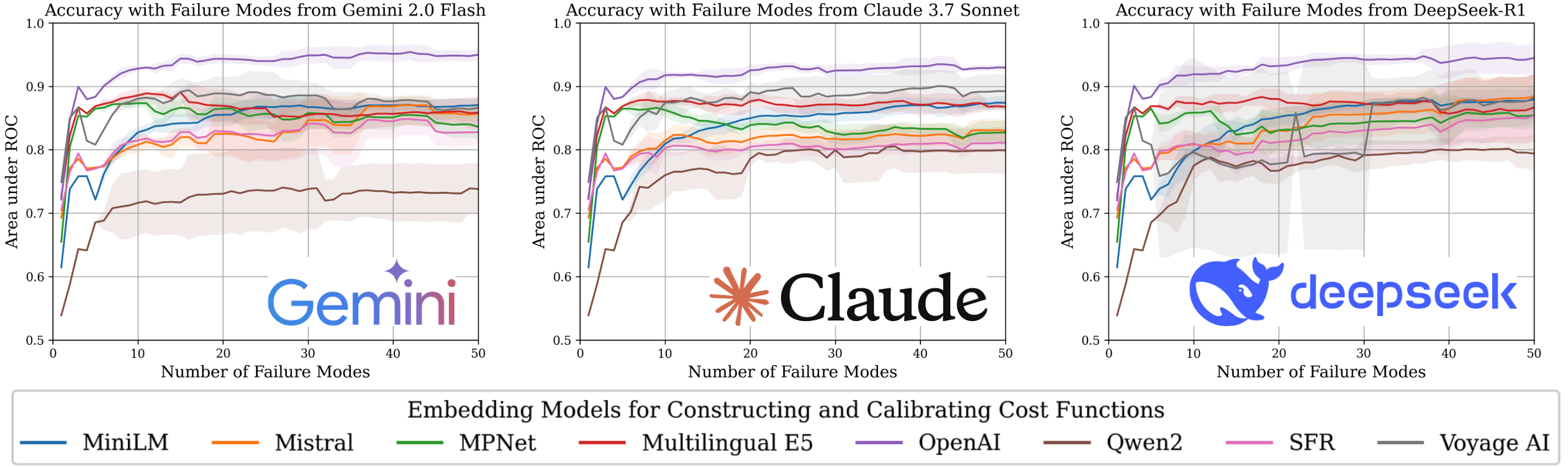}
    \caption{\footnotesize \textbf{Increasing Failure Coverage from Prompting Reasoners:} We prompt Gemini 2.0 Flash, Claude 3.7 Sonnet, and DeepSeek-R1 reasoners to anticipate 1 to 50 failure modes and calibrate cost functions based on eight embedding models to classify if semantically unsafe drone descriptions are closer to failure modes than the distance from the safe drone data to these modes. We use AUROC to measure classifier performance. Results for Llama 4 Scout, GPT-4o, Mercury, and OpenAI's other reasoning models are in Appendix Section~\ref{sec:MoreReasoningModels}.}
    \label{fig:fm_classify}
\vspace{-0.1in}
\end{figure}

\vspace{-0.1in}
\subsection{Embedding classification}
\vspace{-0.1in}

We test \textbf{H1} by measuring the accuracy of our approach in extracting semantic safety constraints. In particular, we create synthetic datasets in several domains such as aerial vehicle drone taxis, self-navigating maritime vessels, and autonomous vehicles. We compile synthetic training data for semantic state descriptions $\Omega^{\text{train}}_s$ that have safe collections of environments and concepts for each domain that robots are known to handle in their nominal operations safely. We also construct safe description test data $\Omega^{\text{test}}_s$ different from those in training data, as well as dangerous description test data $\Omega^{\text{test}}_d$ that contain unsafe collections. We evaluate and calibrate eight text embedding models for classification: MiniLM (22M) ~\citep{wang2020minilm}, Mistral (7B)~\citep{wang2023improving, jiang2023mistral}, MPNet (110M)~\citep{song2020mpnet, reimers2019sentence}, Multilingual-E5 (560M)~\citep{wang2024multilingual}, OpenAI Text Embedding Large~\citep{OpenAIEmbeddingModels2024}, Qwen2 (7B)~\citep{li2023towards}, Salesforce (SFR) Embedding Mistral (7B)~\citep{SFRAIResearch2024}, and Voyage AI's voyage-3-large~\citep{voyageai_embeddings}.

In the first set of experiments, we define around $10$ high-level descriptions of failure modes for each domain. These descriptions are kept concise to cover a wide range of potential unsafe scenarios. We calibrate thresholds for each model and domain using the safe dataset $\Omega^{\text{train}}_s$. We present accuracy results on the total testing dataset $\Omega^{\text{test}}_s \ \cup \ \Omega^{\text{test}}_d$ for the Autonomous Boat and Vehicle synthetic datasets in Table~\ref{tab:10acc}. We additionally compare with an ablation that considers only a single mode ``Safe" and prompting (slow) reasoning models. The results demonstrate that for each domain, \texttt{FORTRESS}'s approach can achieve a balanced accuracy performance higher than $0.90$, effectively distinguishing between safe and unsafe scenes. The other approaches perform poorly since they reason about semantic safety too abstractly (i.e. with 1 mode ``Safe") or too fine-grained (i.e. reasoning models can over-extrapolate from scene descriptions).

%\vspace{-0.1in}
\subsection{Accuracy of proactively anticipating failure modes }
%\vspace{-0.1in}

We verify \textbf{H2} by exploring automating failure mode generation by querying LLM with a prompt that specifies the robot's setting and asks for potential failure modes to anticipate. Effectively, we employ reasoning models to scope out the ``adversarial" strategies. Note that this querying process can be performed during the execution of the nominal trajectory of the robot when it enters a new deployment region or offline if the environment is already known. Fig.~\ref{fig:fm_classify} shows results of querying various reasoning models for up to $50$ failure modes and measuring the performance of the classification performance of the calibrated embedding model cost functions. By querying models for more failure modes, we observe the Area Under the Receiver Operating Characteristic (AUROC) curves, which capture semantic safety classification performance, generally improve and stabilize. OpenAI and Voyage AI embedding models achieve increases in accuracy, with AUROC values above $0.9$, outperforming those generated from manual failure modes. Overall, \texttt{FORTRESS} can effectively identify failure modes to improve coverage of semantically unsafe OOD failure descriptions and avoid suggestions of spurious modes that exacerbate false positive rates.

We also test the classification accuracy of our approach on data we collect from deploying an ANYmal robot (Fig.~\ref{fig:anymal}) in a room under construction. We feed a prompt and an image of the environment to Claude 3.7 Sonnet and extract around $10$ failure modes. We calibrate thresholds on a deployment with relatively safe and manageable concepts detected with an object detector (e.g. ladders, paint cans, and boxes), and measure semantic safety detection rate when deploying the ANYmal near both safe concepts and unsafe ones (e.g. person on a ladder, caution tape, and cables). We present results in the rightmost column group of Table~\ref{tab:10acc} and in Appendix Sections~\ref{sec:MoreROC} \&~\ref{sec:ANYmalApp}. 

%\footnote{ANYmal video demo link at \textcolor{flodarkpurple}{https://www.youtube.com/watch?v=xU-egPQjkFo}}

\begin{table*}[!t] % Use figure* for spanning two columns in a two-column document
\centering
\begin{minipage}{0.4\textwidth} % Adjust width as needed for spacing
  \centering
  \includegraphics[width=\linewidth]{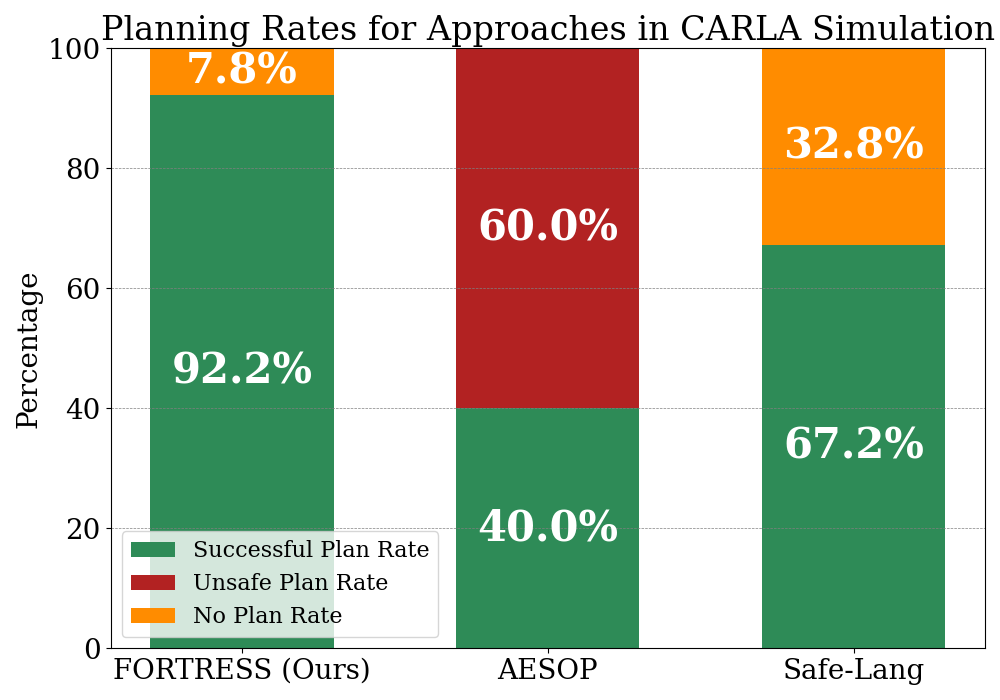} % Replace example_image.jpg with your image file
    {\footnotesize \captionof{figure}{ Planning rates of \texttt{FORTRESS} versus AESOP~\citep{sinha2024real} and Safe-Lang~\citep{santos2024updating} for drone robot in CARLA simulation. We augment baselines with our VLM goal identification for fair comparison.}\label{fig:planrate}}
\end{minipage}%
\hspace{0.5cm} % Add horizontal space to push the image to the right
\begin{minipage}{0.55\textwidth} % Adjust width as needed for spacing
  \centering
  {\scriptsize
    {Computation Times for \texttt{FORTRESS} on Drone Hardware}
    \begin{tabular}{ccc}
        \hline
        Component & Mean Time (s) & STD Dev. (s) \\
        \hline
        Query Molmo for Goal points & 5.82 & 0.13 \\
        Query 2.0 Flash for Failure Modes  & 3.68 & 0.33 \\
        Query 3.7 Sonnet for Failure Modes  & 15.65 & 0.51 \\
        Query R1 for Failure Modes & 12.76 & 1.22 \\
        Calibrate Voyage AI Embeddings & 3.91 & 0.08 \\
        Calibrate OpenAI Embeddings & 5.43 & 0.19 \\
        \hline
        Safety Reasoning Inference  & 0.011 & 0.006 \\
        Reach-Avoid Planner & 1.28 & 0.55 \\
        \hline
    \end{tabular}}
    {\footnotesize {\caption{ Component times of \texttt{FORTRESS} algorithm during quadrotor drone hardware deployment with Jetson Nano. Px4 manages the communication interface, and motion capture sensors enable drone localization. The first grouping of rows is computations done at a low frequency during nominal trajectory or offline; the second grouping is executed during the safety-critical moment when the response is needed. }\label{tab:Times}}}
\end{minipage}
%\vspace{-0.1in}
\end{table*}

%\vspace{-10pt}

%\vspace{-0.1in}
\subsection{Safe fallback planning in real time}
%\vspace{-0.1in}

We measure the success and safety of the generated plans for a single fallback strategy of landing on buildings in the CARLA simulation and hardware experiments on a drone to validate \textbf{H3}. In CARLA (Fig.~\ref{fig:molmo}), we create an agent with the dynamics of a drone to navigate over an urban landscape with concepts like firetrucks, people, cars, traffic cones (more details of setup are in Appendix Section~\ref{sec:CARLADetails}). We compare our approach with recovery planning based on AESOP~\citep{sinha2024real}, which does not avoid semantically unsafe regions in its fallback stage, and adapting Safe-Lang~\citep{santos2024updating} to have a VLM identify unsafe regions for its na\"ive avoidance approach. In Fig.~\ref{fig:planrate}, \texttt{FORTRESS} improves the success rate and safety of generated fallback plans due to its nuanced reasoning that can accurately demarcate semantically safe and unsafe regions. We additionally show how \texttt{FORTRESS} can adaptively change strategies and replan in dynamic environments with moving concepts in Appendix Section~\ref{sec:Dynamic}.

We also deploy our framework on drone hardware to test whether \texttt{FORTRESS} can operate in real-world settings: \texttt{FORTRESS} determines fallback goals for strategy of landing on building roofs, extract semantic safety constraints like High Temperature and Construction, and rapidly generate semantically safe plans seen in Fig.~\ref{fig:herofig}. We measure each component of our approach on Jetson Nano hardware shown in Table~\ref{tab:Times}. The safe embedding reasoner and the reach-avoid planner operate rapidly and therefore are deployed instantaneously once the runtime monitor has triggered a fallback response. While our fallback strategy identification and safety calibration modules that query the multi-modal reasoners take longer, we perform these operations at a low frequency during nominal operations (see Sections~\ref{sec:fallback} \& ~\ref{sec:FailReason}), caching potential locations and semantic safety functions and loading when needed. Additional details on hardware experiments are in the Appendix Section~\ref{sec:DroneApp}.

%\footnote{Quadrotor hardware video demo link at \textcolor{flodarkpurple}{https://www.youtube.com/watch?v=a0XZgwoNLos}}

\vspace{-0.1in}
\section{Discussion and Conclusion}\label{sec:DC}
\vspace{-0.1in}

We present \texttt{FORTRESS}, a framework that prevents OOD failures by bridging open-world reasoning with dynamics-aware planning and control to generate fallback strategies. At a low frequency in nominal operations or offline, \texttt{FORTRESS} uses foundation models to identify fallback goals, anticipate failure modes, and calibrate semantic safety reasoners, caching the goals and semantic cost functions to mitigate latency in safety-critical moments. At runtime, it adaptively synthesizes fallback plans via reach-avoid analysis guided by these goals and semantic constraints. We validate our semantic safety reasoners' detection of potential OOD failures on synthetic maritime and vehicle datasets as well as real-world data from an ANYmal robot. We deploy \texttt{FORTRESS} on drone hardware and in simulation, with improved planning success and safety.

%===============================================================================

\clearpage
% The acknowledgments are automatically included only in the final and preprint versions of the paper.

\section{Limitations}

Our proposed framework currently has limitations that open several promising avenues for future work. Although we currently write semantic fallback strategies for robots and environments, extensions could develop methods that automatically extract semantic descriptions, infer the relevant fallback goals, and adapt them across diverse robotic platforms and operation settings. Non-static fallback strategy goals can also be investigated depending on the robot and environment, such as merging into a different lane with moving traffic in the autonomous vehicle setting. Furthermore, our mapping of semantically unsafe regions to static, fixed‑radius avoidance regions could be extended to handle dynamic or context‑dependent constraints such as adaptive avoid boundaries, dynamic obstacles, terrain‑specific factors including varying surface friction, or environmental disturbances like wind to enable more nuanced reasoning about when and how failures arise. These semantic fallback goals and failures can be retrieved from rule books with high-level specifications on handling the presence of semantically unsafe situations (e.g. a handbook on aviation, naval, and traffic regulations). Finally, while in this paper we focus on preventing OOD failures, future work can expand on our framework for diagnosis and fallback strategies that manage the safety and recovery of robots actively experiencing failures (e.g. recovery with minimal damage during a collision or from a fire).

\acknowledgments{We thank Stanford Robotics Center for their assistance with the ANYmal robot experiments. This work is supported by the NASA University Leadership Initiative (grant \#80NSSC20M0163), Torc Robotics, Swiss Re, and Toyota Research Institute. DM is supported by a NASA Space Technology Graduate Research Opportunity. This article solely reflects the opinions and
conclusions of its authors and not any aforementioned supporting entity. All authors are members of the Autonomous Systems Laboratory (ASL) at Stanford University; LDL is an external affiliate of ASL.}

%===============================================================================

%\bibliography{ref}

\newpage

\tableofcontents

\begin{appendix}

\iffalse
\begin{theorem} Let $\eta>0$ be the worst case error bound of the path tracking control algorithm, and let $\hat\theta_h,  h\in\Phi\cup \{c\}$ be the same formulation as $\theta_h$ except the corresponding distance thresholds are inflated to $\hat l_h=l_h+\eta'$ where $\eta'>\eta$, and let $\hat{\mathcal{X}}\subseteq\mathcal{X}$ be the set of states where $\max_{h \in \Phi \cup \{c\}} \hat\theta_h(x) \leq 0$. Furthermore, suppose a motion planner is employed in state space $\hat{\mathcal{X}}$ to start from point $b$ and reach control-invariant region $\mathcal{B}_\rho(g)$, where $\rho>\eta$, using incremental step planner size of $\Delta\hat x < \min(\rho-\eta, \min_{h\in\Phi\cup\{c\}}2\sqrt{(\eta'-\eta)^2 + 2(l_h+\eta)(\eta'-\eta)})$ produces a trajectory \( \hat\tau = [\hat x_1, \hat x_2, \dots, \hat x_k] \). Then when the path tracking control algorithm is used to follow $\hat\tau$, it produces a trajectory \( \tau^* = [x^*_1, x^*_2, \dots, x^*_k, \dots] \) that safely and successfully executes a fallback strategy (i.e. solves optimization \eqref{eq:optim} with $\Theta^*\leq 0$).
\end{theorem}
\vspace{-13pt}
\begin{proof}
    Proof in Appendix.
\end{proof}
\vspace{-13pt}
\fi

\section{Notation}

\begin{tabular}{ll}
\textbf{Notation} & \textbf{Description} \\
\midrule
$x$ & Robot State \\
$\mathcal{X}$ & State Set \\
$u$ & Control action \\
$\mathcal{U}$ & Control Set \\
$f$ & Dynamics \\
$\Omega_s$ & Set of safe, nominal semantic state descriptions \\
$N$ & number of descriptions in of $\Omega_s$ \\
$\Omega_d$ & Set of semantically unsafe semantic state descriptions \\
$\omega$ & a semantically unsafe semantic state description \\
$\Sigma$ & Set of semantic description of fallback strategies \\
$\sigma_i$ & $i^{th}$ semantic description of in $\Sigma$ \\
$\mathcal{G}_i$ & set of 3D goal coordinates from querying for strategy $\sigma_i$ from VLM \\
$g$ & a 3D goal coordinate \\
$\Phi$ & Set of semantic failure modes \\
$\phi$ & a failure mode \\
$l_\phi$ & physical distances to keep from failure mode $\phi$ and physical objects\\
$l_c$ & physical distances to keep from failure mode $\phi$ and physical objects\\
$\texttt{nearby}(x,l)$ & function that returns state description of concepts in radius $l$ around $x$\\
$\theta_c(x)$ & collision hazard cost function returning positive iff $x$ is within $l_c$ of physical obstacle\\
$\texttt{Embed}(\omega)$ & function Text embedding model producing a vector from description $\omega$\\
$\mathcal{E}_s$ & Set of embeddings vectors of descriptions in $\Omega_s$ \\
$e_\phi$ & embedding vector of failure $\phi$ \\
$\texttt{sim}(e_i,e_\phi)$ & similarity cost function based on cosine similarity of vectors $e_i$ and $e_\phi$ \\
$\Delta_\phi$ & failure embedding similarity threshold calibrated for $\phi$ on safe data $\Omega_s$ \\
$\alpha$ & quantile for threshold (e.g. values $0.0, 0.02, 0.05, 0.1, 0.25, 0.3, 0.4,$ etc) \\
$\theta_\phi(x)$ & semantic safety cost function detecting if state $x$ is in region that could experience $\phi$ \\
$\Phi\cup\{c\}$ & all failure modes and the collision hazard (used when representing cost functions and distances) \\
$\rho$ & radius around goal coordinate determining reach region \\
$\mathcal{B}_\rho(g)$ & ball of radius $\rho$ around goal $g$ \\
$b$ & beginning point for plan/location where fallback response was triggered \\
$\tau$ & trajectory plan \\
$T$ & horizon/steps in trajectory  \\
$x_{\{1:T\}}$ & sequence of states $\{x_1,x_2,...x_T\}$\\
$\mathbb{N}^{\leq T}$ & set of natural numbers from 1 to $T$ \\
$\Omega^{\text{train}}_s$ & training/calibrating set of safe, nominal semantic state descriptions\\
$\Omega^{\text{test}}_s$ & testing set of safe, nominal semantic state descriptions\\
$\Omega^{\text{test}}_d$ & testing/validating set of unsafe/failure semantic state descriptions\\
$\Lambda$ & inverse covariance matrix of safe embedding vector set $\mathcal{E}_s$
\end{tabular}

\newpage

\section{Evaluation of Semantic Safety Cost Functions}

\subsection{Additional Experiments}\label{sec:MoreROC}

\begin{figure*}[h]
    \centering
    \includegraphics[width=\textwidth]{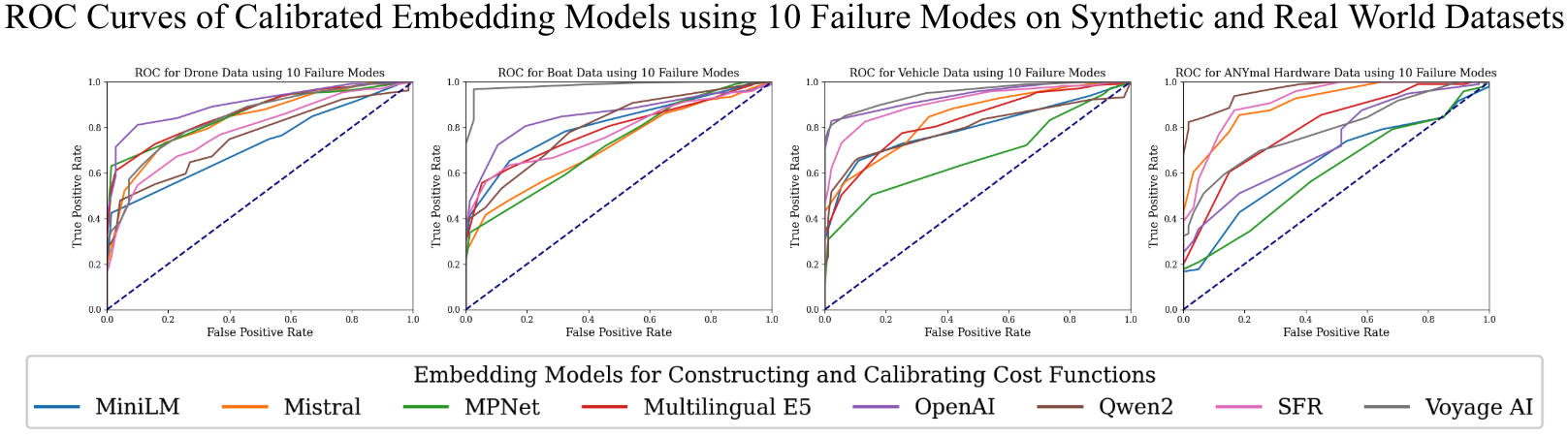}
    \caption{ROC curves using around $10$ failure modes with varying percentile $\alpha$ thresholds on autonomous drones, boats, and vehicle environments using cosine similarity on $8$ embedding models.}
    \label{fig:rocsnorm}
\end{figure*}

\begin{figure*}[h]
    \centering
    \includegraphics[width=\textwidth]{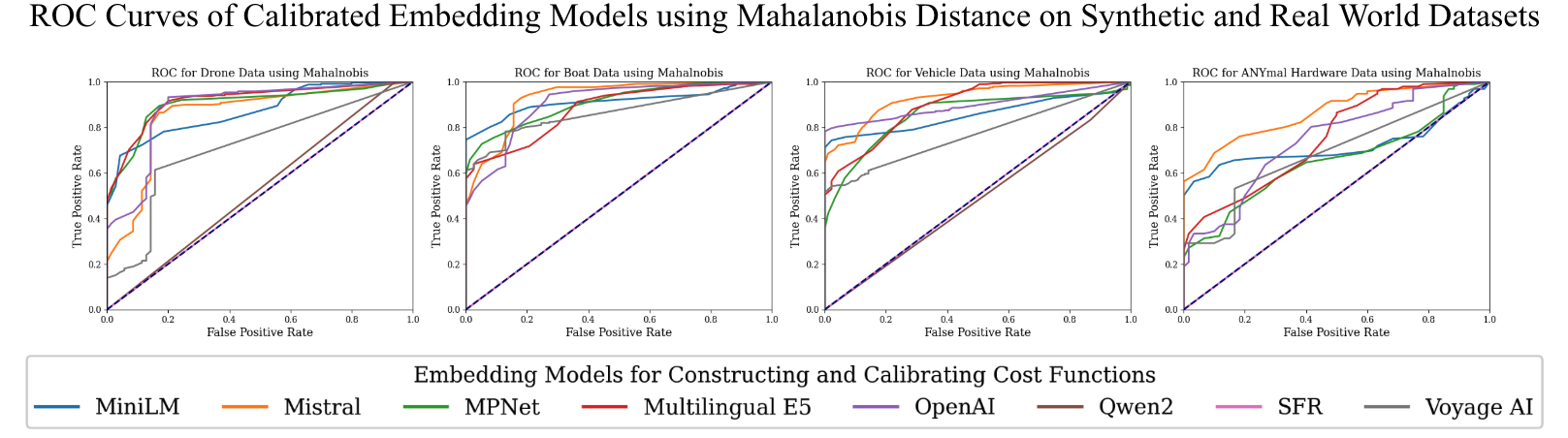}
    \caption{ROC curves using around $10$ failure modes with varying percentile $\alpha$ thresholds on autonomous drones, boats, and vehicle environments using Mahalanobis distance calibrated on cosine similarity on $8$ embedding models.}
    \label{fig:rocsmahalanobis}
\end{figure*}

\begin{figure*}[h]
    \centering
    \includegraphics[width=\textwidth]{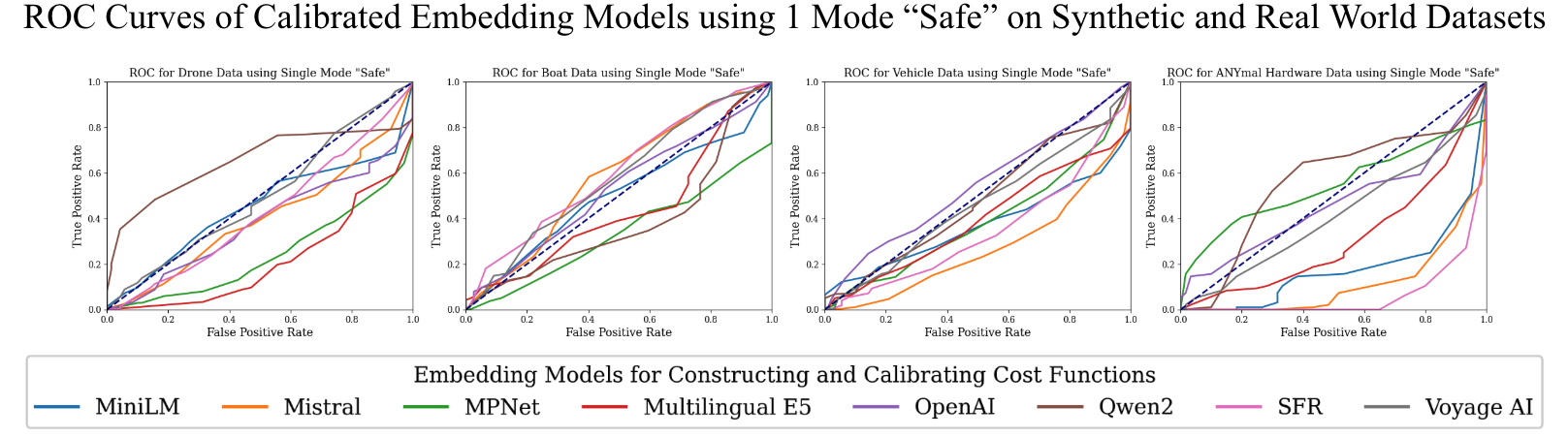}
    \caption{ROC curves using only the ``Safe" Mode with varying percentile $\alpha$ thresholds on autonomous drones, boats, and vehicle environments using cosine similarity on $8$ embedding models.}
    \label{fig:rocssafe}
\end{figure*}

We perform a more comprehensive analysis on our semantically safety cost functions to extract OOD failures. Specifically, we present ROC curves for multiple datasets and calibrate various embedding models. Different classification accuracies are given by varying the quantile threshold alpha in calibration defined in Section~\ref{sec:FailReason}. Our benchmarks include synthetic datasets for autonomous drones, boats, and vehicles. We also present results on real-world data we collected from ANYmal robot deployments.

\emph{Our approach with cosine similarity}: For the main \texttt{FORTRESS} classifier, we construct failure modes per domain and use cosine similarity in the embedding space to measure proximity to each mode (see Fig.~\ref{fig:rocsnorm}). For the ANYmal dataset, we query Claude Sonnet 3.7 and extract 10 failure modes.

\emph{Our approach with Mahalanobis}: We also explore alternative similarity measures and implement a variant inspired by the Mahalanobis distance. We compute the inverse covariance matrix $\Lambda$ of the safe dataset embeddings $\mathcal{E}_s$. Given a new scene embedding $e_t$, the similarity distance function to embedding failure $e_\phi$ is given by: 
\begin{equation}
\texttt{sim}_{\texttt{mahalanobis}}(e_t, e_\phi) := \sqrt{(e_t-e_\phi)^{\top}\Lambda (e_t - e_\phi)}.
\end{equation}
Using this similarity metric, we calibrate the corresponding threshold $\Delta^{\texttt{mahalanobis}}_\phi$ for our cost function. The results of the classification performance for the models on the various datasets are shown in Fig.~\ref{fig:rocsmahalanobis}. The curves for Mahalanobis distance do not show a significant improvement in accuracy compared with just cosine similarity.

\emph{Ablation with single mode}: As an ablation baseline, we employ a single mode called ``Safe." We compute the cosine similarity between each scene embedding and the ``Safe" embedding, and classify scenes with closer similarity (i.e. lower $\texttt{sim}$ values) as safe (opposite of our failure mode decision rule in~\eqref{eq:calib}). Results in Fig.~\ref{fig:rocssafe}.

%\newpage
\subsection{Prompts for Embedding Models}

%\begin{mdframed}[backgroundcolor=light-gray, roundcorner=10pt,leftmargin=1, rightmargin=1, innerleftmargin=15, innertopmargin=15,innerbottommargin=15, outerlinewidth=1, linecolor=light-gray]
%\begin{lstlisting}[linewidth=\columnwidth,breaklines=true]

For autonomous drones:
\begin{tcolorbox}[colback=lightgray!30!white,colframe=lightgray!30!white,sharp corners]
\begin{lstlisting}[linewidth=\columnwidth,breaklines=true]
I am the runtime monitor for a vision-based autonomous vertical takeoff and landing (VTOL) drone operating as an urban air taxi. My job is to analyze the observations of the VTOL and identify whether certain points along the VTOL's trajectory are prone to failure rendering the VTOL unable to handle them as part of its nominal operations. The VTOL's observations take the form of object detections generated by an onboard perception pipeline.

Scene Description: The surroundings of the VTOL include: 
${CONCEPT_LIST}
\end{lstlisting}
\end{tcolorbox}

For autonomous boats:
\begin{tcolorbox}[colback=lightgray!30!white,colframe=lightgray!30!white,sharp corners]
\begin{lstlisting}[linewidth=\columnwidth,breaklines=true]
I am the runtime monitor for a vision-based autonomous boat transporting cargo. My job is to analyze the observations of the boat and identify whether certain points along the boat's trajectory are prone to failure rendering the boat unable to handle them as part of its nominal operations. The autonomous boat's observations take the form of object detections generated by an onboard perception pipeline.

Scene Description: The surroundings of the boat include: 
${CONCEPTS_LIST}
\end{lstlisting}
\end{tcolorbox}

For autonomous vehicles:
\begin{tcolorbox}[colback=lightgray!30!white,colframe=lightgray!30!white,sharp corners]
\begin{lstlisting}[linewidth=\columnwidth,breaklines=true]
I am the runtime monitor for a vision-based autonomous vehicle operating as an urban ground taxi cab. My job is to analyze the observations of the vehicle and identify whether certain points along the vehicle's trajectory are prone to failure rendering the autonomous vehicle unable to handle them as part of its nominal operations. The vehicle's observations take the form of object detections generated by an onboard perception pipeline.

Scene Description: The surroundings of the autonomous vehicle include:
${CONCEPTS_LIST}
\end{lstlisting}
\end{tcolorbox}

For ANYmal robot:
\begin{tcolorbox}[colback=lightgray!30!white,colframe=lightgray!30!white,sharp corners]
\begin{lstlisting}[linewidth=\columnwidth,breaklines=true]
I am the runtime monitor for a vision-based autonomous quadruped ANYmal robot operating in a construction area. My job is to analyze the observations of the quadruped and identify whether certain points along the quadruped's trajectory are prone to failure for the robot or danger caused to the environment. The quadruped ANYmals's observations 
take the form of object detections generated by an onboard perception pipeline.

Scene Description: The surroundings of the quadruped include: 
${CONCEPTS_LIST}
\end{lstlisting}
\end{tcolorbox}

\subsection{Baseline implementation of on-the-fly prompting of slow reasoning models}

As shown in Table~\ref{tab:10acc} of the main paper, we also compare our approach with on-the-fly prompting of the reasoning models. In particular, we directly query the foundation model reasoners with a structured prompt to analyze whether or not a scene is safe. This would theoretically be queried during the safety critical moment to understand the semantically unsafe regions, so we term this as on-the-fly prompting. Table~\ref{tab:10acc} in the main paper and Table~\ref{tab:OTFTimes} below show that simply querying the reasoning models results in poor accuracy and is much slower ($>50\times$) than \texttt{FORTRESS} in analyzing semantic safety. An example prompt is:
\begin{tcolorbox}[colback=lightgray!30!white,colframe=lightgray!30!white,sharp corners]
\begin{lstlisting}[linewidth=\columnwidth,breaklines=true]
I am the runtime monitor for a vision-based autonomous vertical takeoff and landing (VTOL) drone operating as an urban air taxi. My job is to analyze the observations of the VTOL and identify whether certain points along the VTOL's trajectory are prone to failure rendering the VTOL unable to handle them as part of its nominal operations. The VTOL's observations take the form of object detections generated by an onboard perception pipeline.

Scene Description: The surroundings of the VTOL include: 
${CONCEPT_LIST}
Is this situation safe for the agent? Respond in the following format:
Answer: {Yes/No}
\end{lstlisting}
\end{tcolorbox}

\begin{table*}[h] % Use figure* for spanning two columns in a two-column document
    \centering
    \begin{tabular}{ccc}
        \hline
        Reasoning Model & Mean Time (s) & STD Dev. (s) \\
        \hline
        GPT-4o   & 0.49 & 0.03 \\
        o3-mini  & 3.23 & 0.86 \\
        o4-mini  & 2.68 & 0.72 \\
        \hline
    \end{tabular}
    {\footnotesize {\caption{ Inference times of on-the-fly prompting of reasoning models. }\label{tab:OTFTimes}}}
\end{table*}

\newpage

\section{Reasoning Models for Failure Mode Generation}\label{sec:MoreReasoningModels}

\begin{figure*}[h]
    \centering
    \includegraphics[width=\textwidth]{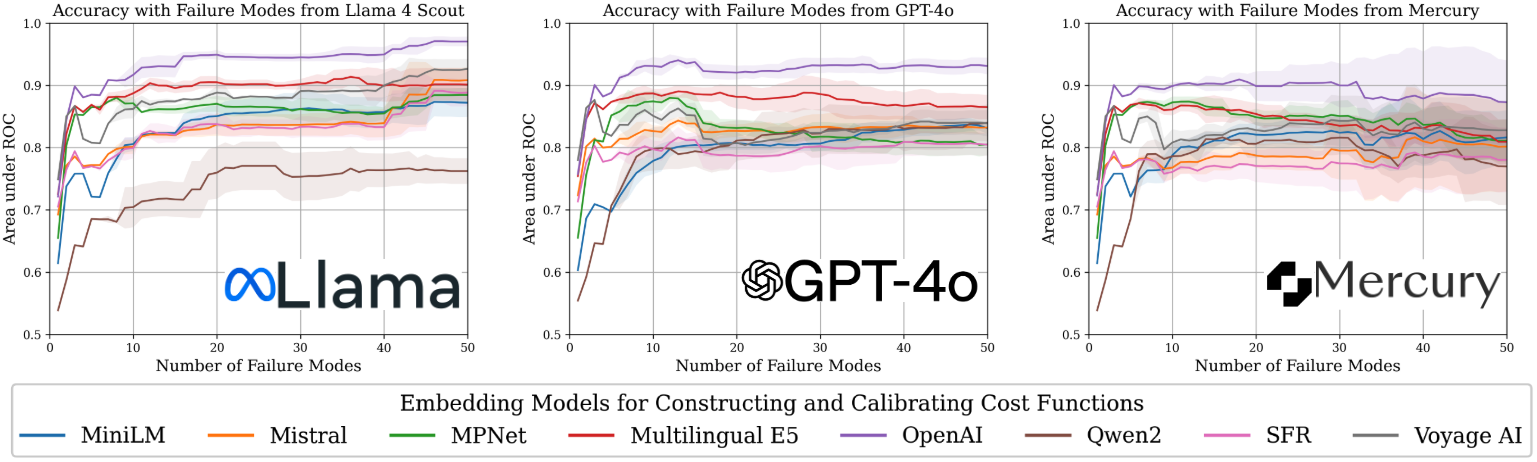}
    \caption{Increasing number of failure modes (1 to 50) taken from prompting Llama 4 Scout, OpenAI GPT-4o, and Mercury reasoners with Area under ROC curve as the measure of classifier model performance queried with eight different embedding models for cost functions over five seeds on the autonomous drone synthetic dataset. }
    \label{fig:otherreasoners}
\end{figure*}

\begin{figure*}[h]
    \centering
    \includegraphics[width=\textwidth]{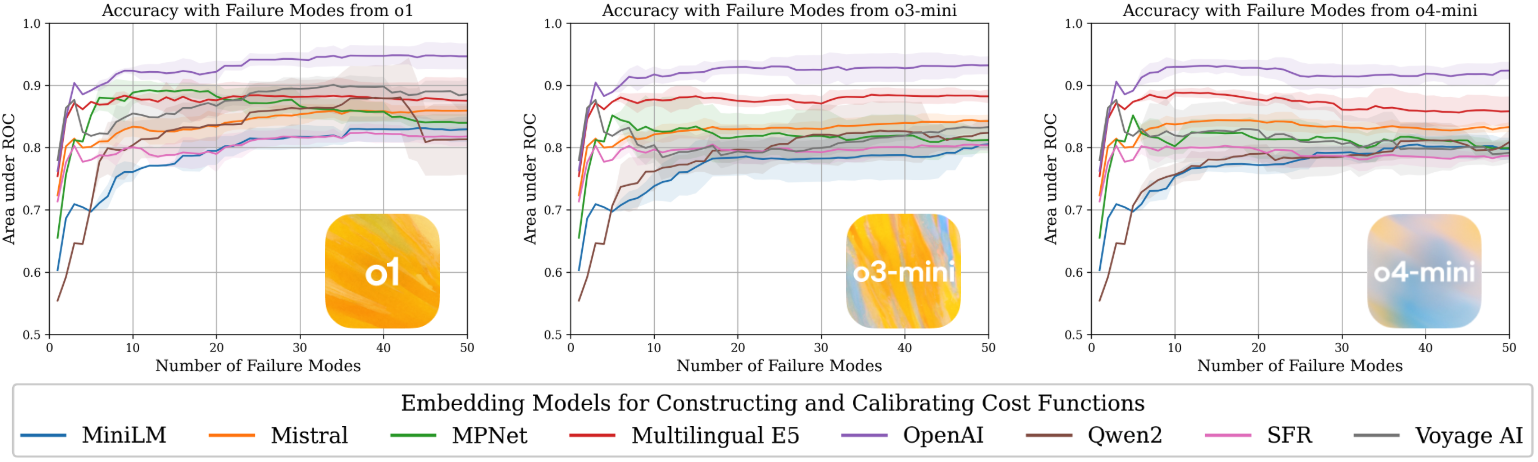}
    \caption{Increasing number of failure modes (1 to 50) taken from prompting o1, o3-mini, and o4-mini which are part of OpenAI's reasoning model series with Area under ROC curve as the measure of classifier model performance queried with eight different embedding models for cost functions over five seeds on the autonomous drone synthetic dataset. }
    \label{fig:openaireasoners}
\end{figure*}

We present results for querying foundation model reasoners for failure modes given scene descriptions and calibrating all the embedding model cost functions with cosine similarity to the failure modes. We present the accuracy of the embedding model cost functions incrementally increasing from 1 to 50 failure modes from each reasoner. We use Area under the ROC curve as the classification performance metric. We present the results on Llama 4 Scout, OpenAI GPT-4o, and Mercury in Fig.~\ref{fig:otherreasoners} and OpenAI reasoning model series in Fig.~\ref{fig:openaireasoners}. For reference, each plot constitutes $8$ embedding models $\times$ $50$ evaluations on varying numbers of failure modes $\times$ $5$ queries for failure mode lists from the reasoning models $= 2000$ calibration experiments for each reasoning model.

Overall, \texttt{FORTRESS} demonstrates strong performance in leveraging foundation models to anticipate OOD failure modes and accurately classify the failure, with the OpenAI text embedding large model consistently achieving over $0.95$ AUROC and Multilinguale-E5 and Voyage AI embedding models sometimes performing over $0.9$. It is worth noting that the AUROC largely holds steady as we include more failure modes. A significant drop would be anticipated if spurious suggestions were causing safe descriptions to be misclassified as unsafe. However, the embedding model cost functions become gradually more accurate with more failure modes generated by the reasoners, suggesting that \texttt{FORTRESS} can employ the reasoners to effectively understand and generate more failure modes that accurately improve failure coverage.

A query and response from the Llama 4 Scout model look like:
\newpage
%\begin{mdframed}[backgroundcolor=light-gray, roundcorner=10pt,leftmargin=1, rightmargin=1, innerleftmargin=15, innertopmargin=15,innerbottommargin=15, outerlinewidth=1, linecolor=light-gray]

\begin{tcolorbox}[colback=lightgray!30!white,colframe=lightgray!30!white,sharp corners]
\begin{lstlisting}[linewidth=\columnwidth,breaklines=true]

Query: Given an urban view from a drone agent, provide concise descriptions no more than four words of at least 50 general potential external failures that the drone may encounter in this environment. Avoid internal problems/failures of the drone or its equipment. Examples include Low Visibility, Near Human, Near Vehicle in motion, High Temperature, Unauthorized Access.
Provide in the format:
n. failure

Response: Here are 50 potential external failures that the drone may encounter in an urban environment:

1. Low Visibility
2. Near Human
3. Near Vehicle
4. High Temperature
5. Unauthorized Access
6. Power Line Nearby
7. Tall Building Nearby
8. Strong Wind Gust
9. Rainy Weather
10. Snowfall
11. Foggy Conditions
12. Bird Strike
13. Insect Swarm
14. Air Traffic Nearby
15. Emergency Vehicle
16. Pedestrian Crossing
17. Road Construction
18. Bridge Nearby
19. Tunnel Entrance
20. Overhanging Tree
21. Loose Debris
22. Nearby Crane
23. High Voltage Lines
24. Radio Frequency Interference
25. Water Body Nearby
26. Steep Slope
27. Unstable Ground
28. Nearby Fire
29. Toxic Gas Leak
30. Nearby Accident
31. Protests or Riots
32. Large Animal Nearby
33. Narrow Alleyway
34. Billboard or Sign
35. Overhead Wires
36. Construction Site
37. Parked Vehicle
38. Moving Train
39. Boat or Ship
40. Nearby Helicopter
41. Telephone Lines
42. Severe Weather
43. Sandstorm
44. Nearby Explosion
45. Unmarked Obstacle
46. Cycling Pedestrian
47. Dense Vegetation
48. Steep Cliff
49. Water Spray
50. Dense Traffic

Let me know if you'd like me to expand on this list!
\end{lstlisting}
\end{tcolorbox}
%\end{mdframed}

\newpage

\section{CARLA Simulation Experiments}\label{sec:CARLADetails}

We also deploy and evaluate \texttt{FORTRESS} in the CARLA simulator. We model the ego spectator view as a drone agent observing the city from above the buildings. Some of the buildings have people, firetrucks, traffic cones, and cars on their roofs. Given that a runtime monitor has triggered the need for a fallback response in this situation, we are tasked with generating and executing a semantically safe fallback plan.

\subsection{Implementation Details of \texttt{FORTRESS} in CARLA:  }

We first query the VLM Molmo~\cite{deitke2024molmo} for 2D coordinates. For the fallback strategy of landing on a building roof, we might query something like: 

\begin{tcolorbox}[colback=lightgray!30!white,colframe=lightgray!30!white,sharp corners]
{    \centering
    \includegraphics[width=0.5\textwidth]{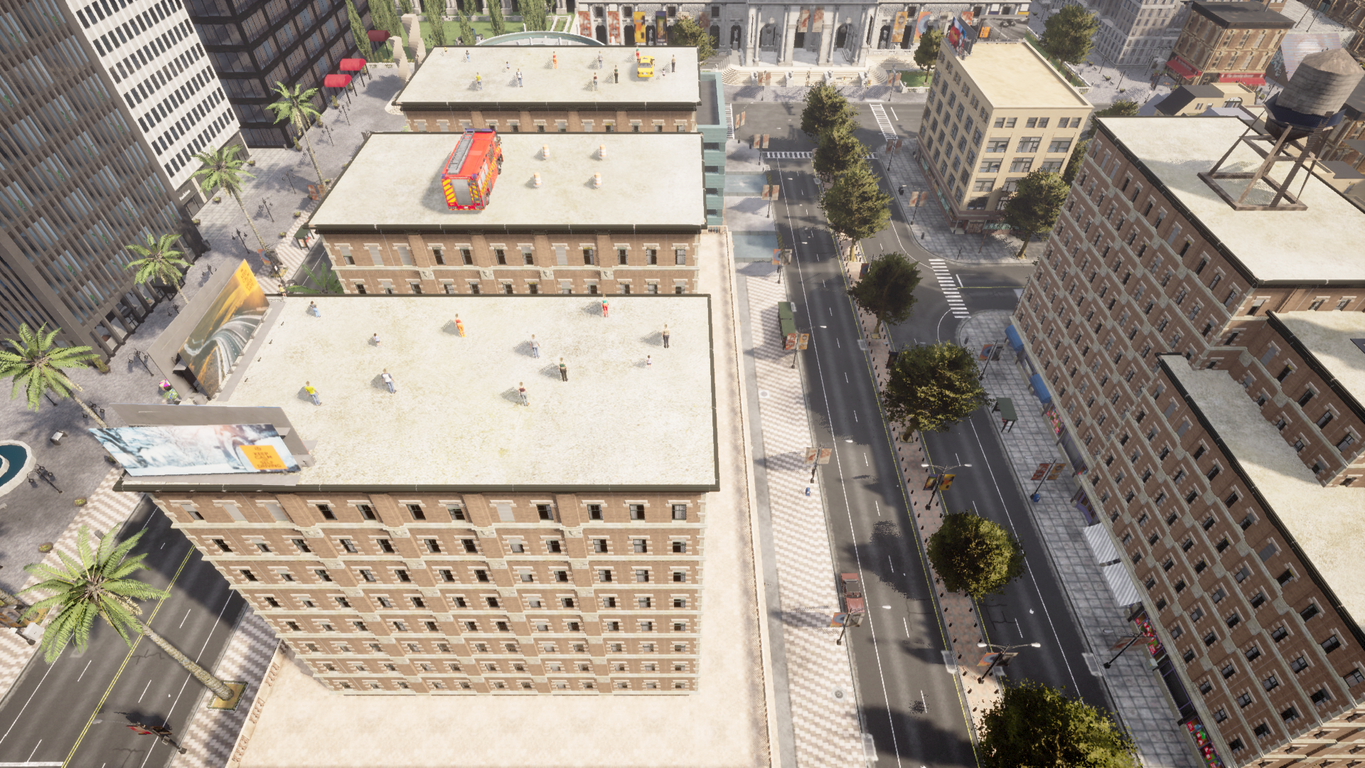}
}
\begin{lstlisting}[linewidth=\columnwidth,breaklines=true]
Query: point to empty, horizontal building roofs.

Response: <points x1="32.4" y1="51.2" x2="38.5" y2="9.5" x3="42.1" y3="24.0" x4="74.6" y4="8.1" x5="89.5" y5="20.6" x6="92.4" y6="53.5" x7="97.5" y7="69.3" alt="empty, horizontal building roofs">empty, horizontal building roofs</points>
\end{lstlisting}
\end{tcolorbox}

The response has 2D xy coordinates (that have been normalized to be in the range 0-100) from which we extract the proposed empty building roof points. They correspond to the following locations:
\begin{tcolorbox}[colback=lightgray!30!white,colframe=lightgray!30!white,sharp corners]

\includegraphics[width=0.5\textwidth]{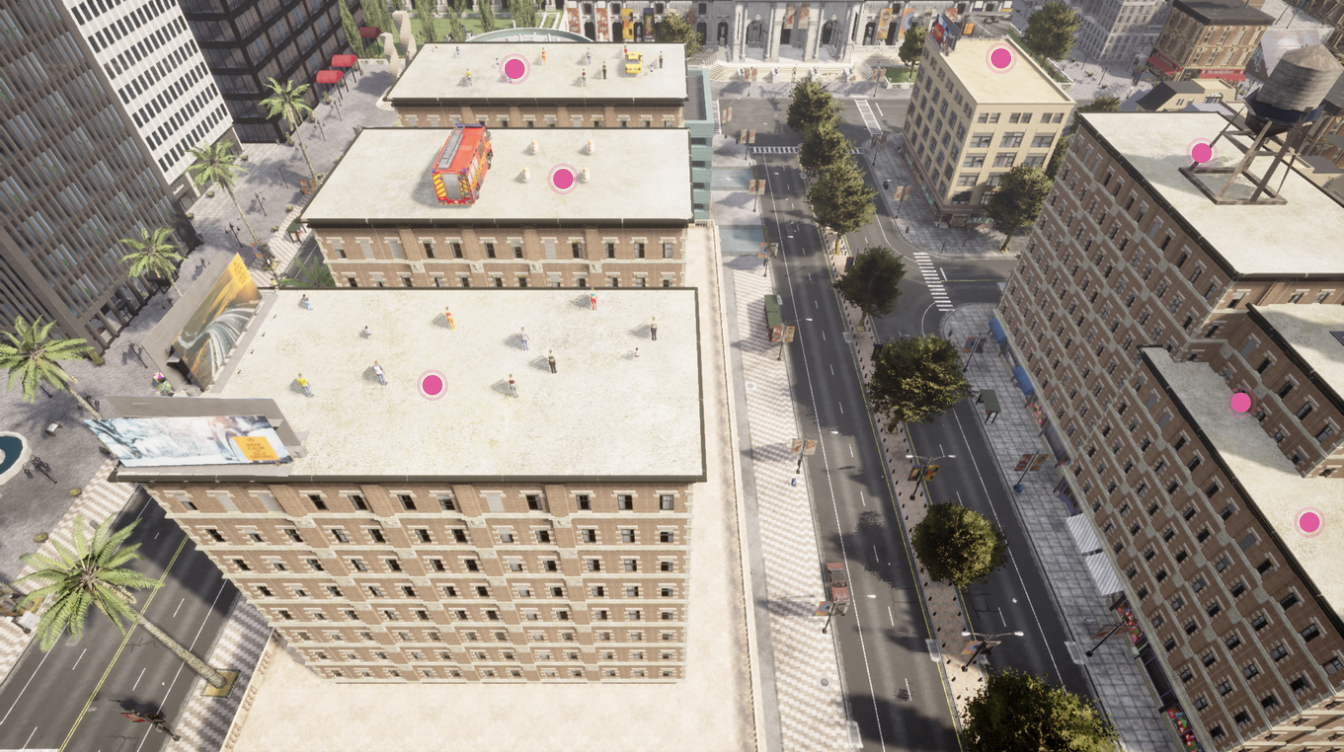}
\end{tcolorbox}
From this we can use CARLA's built-in pose and position estimates, depth view, and camera intrinsics to construct 3D global coordinates from these points, which form our goal points.

For the semantic safety cost function, we employ the OpenAI text embedding model~\cite{OpenAIEmbeddingModels2024} and calibrated with cosine similarity (the accuracy of which is displayed on the leftmost image in Fig.~\ref{fig:rocsnorm}). 

During runtime, we identify the semantically unsafe regions in the state space and perform reach-avoid planning to enter into a 1-meter radius around one of the goals identified by Molmo while avoiding the unsafe regions. Specifically, we use open-vocabulary object detectors YOLOv8~\cite{sohan2024review} and OWLViT~\cite{minderer2022simple} to identify the concepts on the building roofs and depth maps for their locations. We make $l_c$ as 2 meters and $l_\phi$ for all $\phi$ as 4 meters. We employ an RRT planner with an incremental step size of 0.5 meters and LQR to track the planned path. Because the object detectors are not always accurate from far distances, we perform a cycle of path tracking at most 15 steps of the RRT plan while querying the object detectors and replanning a trajectory to the goal from the anticipated location. In the replanning, stage if we find the original goal is infeasible (because there is no semantically safe plan to reach it), we try to plan paths to new nearby goals identified by the VLM.

\subsection{Implementation Details of Baselines in CARLA} 
While there are no works we are aware of directly generate semantically safe fallback plans in real time to prevent OOD failures, we compare our approach with two baselines by adapting from adjacently related works AESOP~\cite{sinha2024real} and Safe-Lang~\cite{santos2024updating}. 

While AESOP focuses on deciding whether a situation is OOD and what semantic fallback response is needed, its fallback planner produces a trajectory to a manually predefined fallback goal without considering the safety of the plan. We use this baseline with a few modifications such as augmenting it with our approach of querying Molmo for fallback goal identification and employing the same planning. We also include na\"ive collision hazard avoidance of keeping 2 meters from physical objects to highlight the comparative improvements provided by our semantic safety cost function. 

Safe-Lang, on the other hand, uses human language input to identify the semantically unsafe regions for the safe fallback policy/plans and does a simple object avoidance for this. Since it is impractical for humans to perform descriptive safety interventions especially in large scale deployments, we model this using the open-vocabulary object detectors to identify objects automatically and perform a blanket avoidance distance of 4 meters. We augment Safe-Lang with our VLM goal identification since their approach also does not identify new fallback strategies and goals when the nominal goal is infeasible. 

To ensure a fair comparison of \texttt{FORTRESS} and the baselines, we used the same planning and path tracking methods (i.e. RRT+LQR) for generating controls for the fallback plans.

\subsection{Discussion on Results}
The results of all the approaches are shown in Fig. 6 of the main paper. Since AESOP does not consider semantic safety, it generates plans that enter into semantically unsafe regions such as near a firetruck with traffic cones or a rooftop party with people. Furthermore, while Safe-Lang's fallback approach avoids objects identified by the object detector, it is unable to distinguish between semantically safe and unsafe concepts and therefore incurs a high no-plan rate (i.e. no safe plan could be found to reach any fallback strategy goal). Our approach can reliably generate safe plans successfully with more than a $90\%$ success rate for the strategy of ``landing on building roofs." The situations when \texttt{FORTRESS} cannot find a safe plan to implement are because all building roofs are occupied by unsafe concepts.

\section{Replanning with Dynamic Concepts}
\label{sec:Dynamic}

\begin{figure*}[h]
    \centering
    \includegraphics[width=\textwidth]{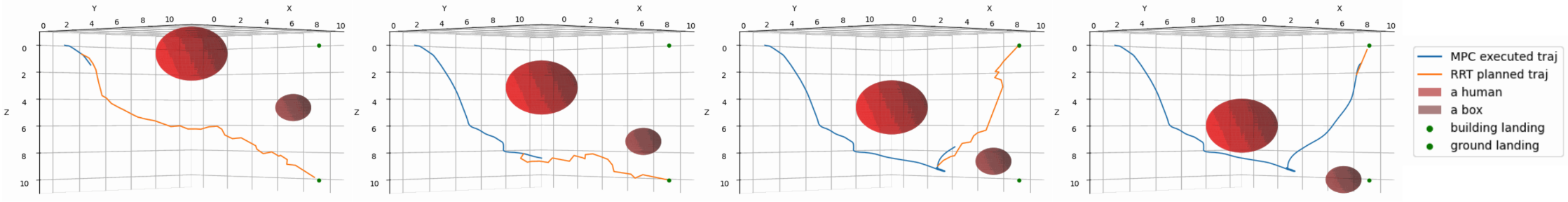}
    \caption{Example demonstrating (re)planning of safe fallback plans with moving objects. \texttt{FORTRESS} originally produced a plan to the first goal point that implemented the strategy of landing on the ground by avoiding the human (skydiver) and box. However as the human and box descended, the original fallback strategy was physically and semantically infeasible so it replanned a path to implement the next strategy of landing on a building.}
    \label{fig:movingobj}
\end{figure*}

We demonstrate an example of how \texttt{FORTRESS} changes plans and implements a new strategy when circumstances evolve to ensure both semantic and physical safety of the fallback. Specifically, we employ a robot with the dynamics of a double integrator along 3 dimensions. The environment consists of a human skydiving and a falling box, and we have two strategies: land on the ground or land on a building. In this simple example, \texttt{FORTRESS} understands the semantically unsafe regions in the state space (i.e. skydiving human is unsafe to be near while the box is relatively safe) as it is changing and in response adjusts its fallback plans or even implements a wholly new strategy. In this setting, we use RRT to replan every 10 steps. Furthermore, we employ MPC for path tracking with a planning horizon of 2 seconds, corresponding to a 20 timestep look-ahead. We use velocity and position to estimate and inflate the avoid/unsafe regions of moving objects.

\section{Details on Hardware}

\subsection{ANYmal Robot Hardware Experiments}\label{sec:ANYmalApp}

We discuss some details on the setup for the experiments conducted on the ANYmal hardware (Generation D) for testing the accuracy performance of \texttt{FORTRESS}'s semantic safety cost function in identifying potential OOD failures. The setting is a room that is actively under construction. We first collect some safe, nominal data with represent concepts that are within the operational capabilities of the ANYmal. Specifically, we navigate the robot around the room with objects like boxes, paper scraps, ladders, paint cans, and people, which we detect using the OWLv2~\cite{minderer2023scaling} open-vocabulary object detector. We extract around failure modes from querying Claude 3.7 with an image of the environment, such as Sharp Objects, Unauthorized Access, Chemical Spill, Unstable Region, Worker Injury, Entanglement, and Slippery region. Then, after calibrating our semantic safety cost functions using the Qwen2 text embedding model, we deploy the robot again but this time we place unsafe objects in the room. As seen in Fig.~\ref{fig:anymalfailuredetect}, this includes a person standing on a ladder (which is unsafe since the ANYmal may cause the person to fall down), cables (which can cause the robot to get entangled and trip), caution tape, and a blue toxic spill. \texttt{FORTRESS} can detect these semantically unsafe concepts even though no similar failures existed in the training dataset. Notice how ``person" and ``ladder" are individually detected as semantically safe but a ``person on a ladder" triggers ``Worker Injury" failure -- entering into this unsafe region can result in destabilizing the person on the ladder and potentially injure them. 

%Full video at \textcolor{flodarkpurple}{https://www.youtube.com/watch?v=xU-egPQjkFo}.

\begin{figure*}[h]
    \centering
    \includegraphics[width=\textwidth]{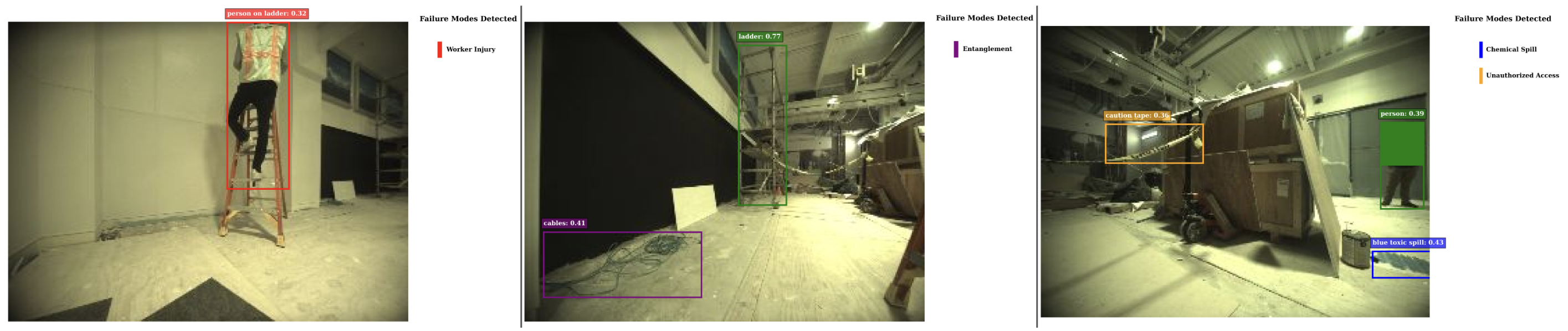}
    \caption{Examples of OOD failures detected by \texttt{FORTRESS} for deployment of ANYmal hardware in a room under construction. The green boxes indicate semantically safe concepts for the robot such as a ladder or a person. The other colors show potential hazards: in the image, the boxes are labeled with what objects are detected and on the legend we list their corresponding failure modes that have been identified by the semantic safety cost functions.}
    \label{fig:anymalfailuredetect}
\end{figure*}

\subsection{Quadrotor Drone Hardware Experiments}\label{sec:DroneApp}

We explain our experiments for implementing \texttt{FORTRESS} on quadrotor drone hardware. We equip our quadrotor with a Jetson nano for computation and an Intel Realsense D435 camera to stream RGBD data. We use the Optitrack motion capture system for localization. We deploy the drone in an environment with buildings and task \texttt{FORTRESS} with producing a response of implementing the fallback strategy of landing on building rooftops. As seen in Fig. 1 of the main paper, the rooftops consist of various safe and unsafe concepts. Specifically, for unsafe regions, we have a building with a ladder and caution tape (depicting a construction), one on fire, and a rooftop parking lot with two vehicles. Additionally, there are two buildings that are relatively safe by themselves such as one with nothing on it and one with a garden rooftop with grass and trees, and the drone lands on the garden rooftop. \texttt{FORTRESS} successfully identifies that the garden rooftop provides a safe spot to implement the strategy of landing on a building while the other buildings are either centers of OOD failures or are very close to them.

% As seen in the demo video at \textcolor{flodarkpurple}{https://www.youtube.com/watch?v=a0XZgwoNLos}

To obtain goal points, we query Molmo for rooftop landing locations as follows:

\begin{tcolorbox}[colback=lightgray!30!white,colframe=lightgray!30!white,sharp corners]
{    \centering
    \includegraphics[width=0.5\textwidth]{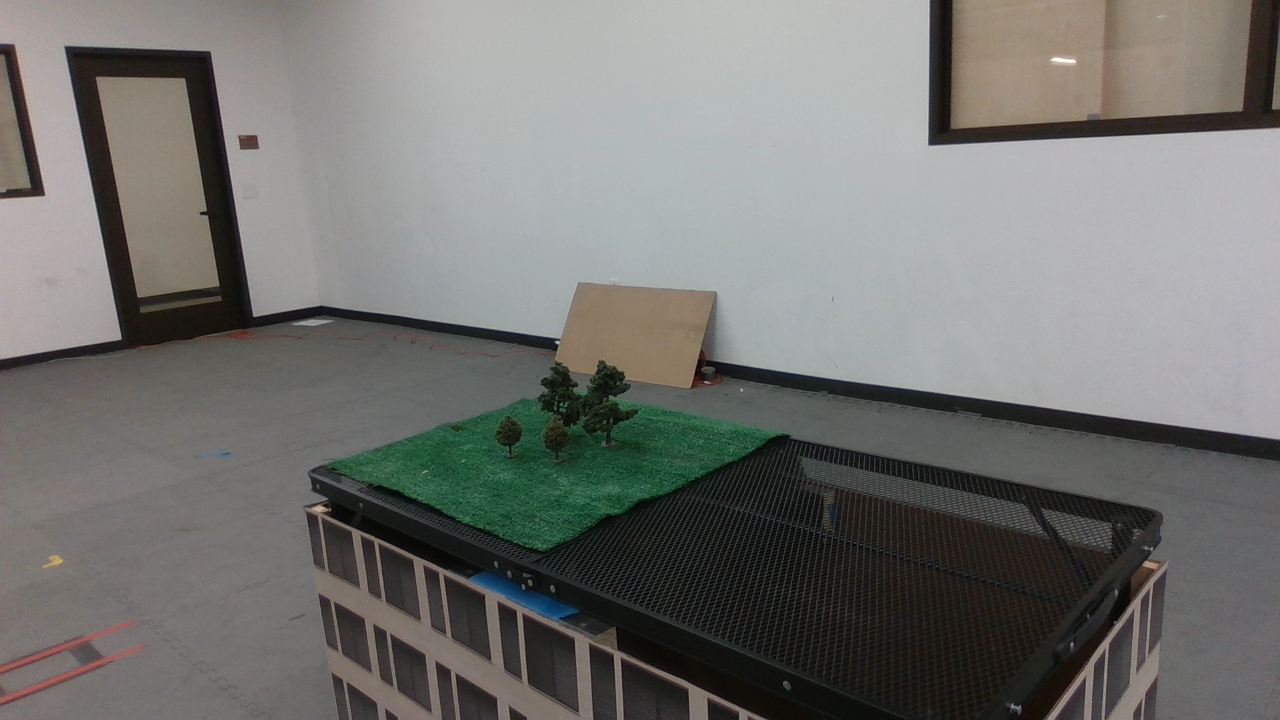}
}
\begin{lstlisting}[linewidth=\columnwidth,breaklines=true]
Query: point to middle of building rooftop.

Response: <point x="67.1" y="71.1" alt="middle of building rooftop">middle of building rooftop</point>
\end{lstlisting}
\end{tcolorbox}
This point corresponds to the pink point in the following image:
\begin{tcolorbox}[colback=lightgray!30!white,colframe=lightgray!30!white,sharp corners]
{    \centering
    \includegraphics[width=0.5\textwidth]{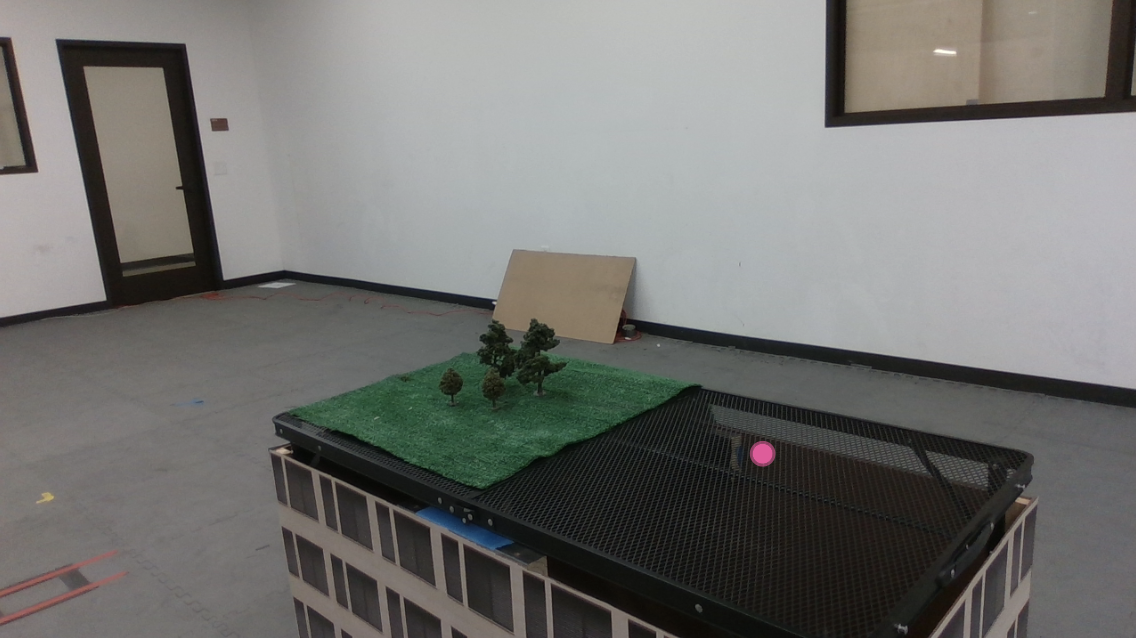}
}
\end{tcolorbox}

Similar to the CARLA setting, we identify the semantically unsafe regions in the state space by employing the OpenAI text embedding model calibrated with cosine similarity. We perform reach avoid planning to enter into a 0.5-meter radius around the goals identified by Molmo while avoiding the unsafe regions. We use the OWLv2~\cite{minderer2023scaling} open-vocabulary object detector to identify the concepts on the building roofs and their locations. We make $l_c$ as 0.1 meters and $l_\phi$ for all $\phi$ as 1.5 meters. We also validate the safety of our fallback goals by filtering out goal points $g$ when $\exists\phi\in\Phi$ where $\hat \theta_\phi(g)>0$. $\hat \theta_\phi$ is the same as $\theta_\phi$ except its distance parameter $l_\phi$ is inflated by 0.5 meters. We employ an RRT planner with an incremental step size of 0.1 meters. We track the plan with interpolation and publish the interpolated waypoints on ROS2 to the PX4 controller at 100Hz.

\section{Theorem on Safely and Successfully Solving Optimization~\eqref{eq:optim}}
\label{sec:PlanningProof}

\begin{utheorem}  Let $\eta>0$ be the worst case error bound of the path tracking control algorithm, and let $\hat\theta_h,  h\in\Phi\cup \{c\}$ be the same formulation as $\theta_h$ except the corresponding distance thresholds are inflated to $\hat l_h=l_h+\eta'$ where $\eta'>\eta$, and let $\hat{\mathcal{X}}\subseteq\mathcal{X}$ be the set of states where $\max_{h \in \Phi \cup \{c\}} \hat\theta_h(x) \leq 0$. Furthermore, suppose a motion planner is employed in state space $\hat{\mathcal{X}}$ to start from point $b$ and reach control-invariant region $\mathcal{B}_\rho(g)$, where $\rho>\eta$, using incremental step planner size of $\Delta x < \min(\rho-\eta, \min_{h\in\Phi\cup\{c\}}2\sqrt{(\eta'-\eta)^2 + 2(l_h+\eta)(\eta'-\eta)})$ produces a trajectory \( \hat\tau = [\hat x_1, \hat x_2, \dots, \hat x_k] \). Then when the path tracking control algorithm is used to follow $\hat\tau$, it produces a trajectory \( \tau^* = [x^*_1, x^*_2, \dots, x^*_k, \dots] \) that safely and successfully executes a fallback strategy (i.e. solves optimization of \eqref{eq:optim} with $\Theta^*\leq 0$).
\end{utheorem}

\begin{proof}
The first and third constraints of \eqref{eq:optim} are satisfied trivially with the planner and path tracking control algorithm. The core proof of the theorem therefore is about guaranteeing that planning and path tracking will satisfy the optimization and second constraint of \eqref{eq:optim}. 

For the optimization, the worst case to consider is when two consecutive waypoints in the safe/valid space $\hat{\mathcal{X}}$ are on the border of the sphere created by $\texttt{nearby}$ for cost/failure $h\in \Phi\cup \{c\}$ using the inflated radii -- this is the closest that the waypoints can be to the epicenter of the unsafe region. In this case, the two waypoints are a distance $l_h+\eta'$ from the point/concept/obstacle. The length of the line segment between the two waypoints is $\Delta x$ which is the planning algorithm's step size parameter. The robot, and therefore the line segment, should be at most $l_h+\eta$ close to the failure point/concept/obstacle since path tracking of the line segment has worst case error of $\eta$. The line segment in the worst case is tangent to the sphere of radius $l_h+\eta$ and its midpoint is on a sphere with the same center but with radius $l_h+\eta$. This creates a right angle triangle from which we obtain the upper bound of $(\Delta x)/2$ which is the distance from a waypoint to the midpoint of the line segment: $(\Delta x/2)^2 + (l_h+\eta')^2 < (l_h+\eta)^2$. Intuitively larger $\Delta x$ would mean less granular planning and therefore more error. This inequality can be rearranged to get $\Delta x < 2\sqrt{(\eta'-\eta)^2 + 2(l_h + \eta)(\eta'-\eta)}$. See the below diagram for a visualization of the geometry. This constraint ensures that if a planner plans a path avoiding any state with positive values from the inflated length cost functions $\hat{\theta}$, then the path tracking algorithm will generate a trajectory that avoids states with positive values from the original length cost functions $\theta$.

\begin{center}

\begin{tikzpicture}
    % Define variables
    \def\lpluse{2.2} % radius of smaller circle (l+η)
    \def\lplusepr{3} % radius of bigger circle (l+η')
    
    % Calculate x coordinate of point A on bigger circle so that B(0, -l+η) and A share same y (horizontal BA)
    \pgfmathsetmacro{\Ax}{sqrt(\lplusepr*\lplusepr - \lpluse*\lpluse)}
    
    % Define points
    \coordinate (O) at (0,0); % center
    \coordinate (B) at (0, -\lpluse); % on smaller circle, vertical below O
    \coordinate (A) at (\Ax, -\lpluse); % on bigger circle, horizontally aligned with B
    \coordinate (C) at (-\Ax, -\lpluse); % on bigger circle, horizontally aligned with B
    
    % Draw circles
    \draw (O) circle (\lpluse);
    \draw (O) circle (\lplusepr);
    
    % Draw triangle
    \draw[thick] (O) -- (B); % vertical segment (radius smaller circle)
    \draw[thick] (O) -- (A); % radius bigger circle
    \draw[thick] (B) -- (A); % horizontal tangent segment
    \draw[thick] (C) -- (A); % horizontal tangent segment
    
    % Mark right angle at B
    \draw pic [draw=black, angle radius=6mm] {right angle = O--B--A};
    
    % Labels for lengths
    \node[left] at (0,-\lpluse/2) {$l+\eta$}; % midpoint OB
    \node[below] at (\Ax/2,-\lpluse) {$\Delta x/2$}; % midpoint BA
    \node[below] at (-\Ax/2,-\lpluse) {$\Delta x/2$}; % midpoint BA
    \node[above] at (1.2*\lplusepr/2,-1.3*\lpluse/2) {$l+\eta'$}; % midpoint OA
    
    % Label points
    %\node[left] at (O) {$O$};
    %\node[below left] at (B) {$B$};
    %\node[below right] at (A) {$A$};
    
\end{tikzpicture}
\end{center}

Furthermore, we need another constraint to ensure that the path tracking trajectory of the plan reaches and remains in the goal. Since we assume region $\mathcal{B}_\rho (g)$ is control invariant, meaning that once we enter it, there are controls that ensure the robot remains within the region, we only need to guarantee the plan definitively enters the region. Once again, since the path tracking error is $\eta$, we simply ensure that the planner's incremental distance does not miss the sphere region $\rho$ around goal $g$. The planner will reach a point within $\Delta x$ of the goal $g$, and the path tracker will be at most $\eta$ from that point. We can ensure the path tracker enters the goal region using the constraint that $\eta+\Delta x <\rho$. Bringing together all the constraints and considering the worst case failure, we get the upper bound of $\Delta x < \min(\rho-\eta, \min_{h\in\Phi\cup\{c\}}2\sqrt{(\eta'-\eta)^2 + 2(l_h+\eta)(\eta'-\eta)})$. Ultimately, this bound of the planning step size ensures that if the planning algorithm is employed in modified state space $\hat{\mathcal{X}}$ and reaches the goal, then the path tracking algorithm will be able to produce a safe trajectory (i.e. avoids the semantically unsafe regions with OOD failures) and implements the fallback strategy by reaching the fallback goal region.
\end{proof}

\end{appendix}

\end{document}